\documentclass[final,12pt]{elsarticle}

\usepackage{amsmath,amssymb,amsthm}
\usepackage{url}
\usepackage{graphicx}
\usepackage{booktabs}
\usepackage{multirow}
\usepackage{array}
\usepackage{xcolor}
\usepackage{mathtools}
\usepackage{float}
\usepackage{algorithm}
\usepackage{algpseudocode}

\makeatletter
\@ifundefined{theorem}{\newtheorem{theorem}{Theorem}}{}
\@ifundefined{proposition}{\newtheorem{proposition}{Proposition}}{}
\@ifundefined{corollary}{\newtheorem{corollary}{Corollary}}{}
\@ifundefined{remark}{\newtheorem{remark}{Remark}}{}
\makeatother

\journal{Neural Networks}

\begin{document}

\begin{frontmatter}

\title{DDCL: Deep Dual Competitive Learning:\\
A Differentiable End-to-End Framework\\
for Unsupervised Prototype-Based Representation Learning}

\author{Giansalvo Cirrincione}
\address{Lab.\ LTI, Universit\'{e} de Picardie Jules Verne, Amiens, France}
\ead{exin@u-picardie.fr}

\begin{abstract}
A persistent structural weakness in deep clustering is the \emph{disconnect}
between feature learning and cluster assignment.
Most architectures invoke an external clustering step, typically $k$-means,
to produce pseudo-labels that guide training, preventing the backbone from
directly optimising for cluster quality.

This paper introduces Deep Dual Competitive Learning (DDCL), the first fully
differentiable end-to-end framework for unsupervised prototype-based
representation learning.
The core contribution is architectural: the external $k$-means is replaced by
an internal Dual Competitive Layer (DCL) that generates prototypes as native
differentiable outputs of the network.
This single inversion makes the complete pipeline, from backbone feature
extraction through prototype generation to soft cluster assignment, trainable
by backpropagation through a single unified loss, with no Lloyd iterations,
no pseudo-label discretisation, and no external clustering step.

To ground the framework theoretically, the paper derives an exact algebraic
decomposition of the soft quantisation loss into a simplex-constrained
reconstruction error and a non-negative weighted prototype variance term.
This identity reveals a self-regulating mechanism built into the loss geometry:
the gradient of the variance term acts as an implicit separation force that
resists prototype collapse without any auxiliary objective, and leads to
a global Lyapunov stability theorem for the reduced frozen-encoder system.

Six blocks of controlled experiments validate each structural prediction.
The decomposition identity holds with zero violations across more than $10^5$
training epochs; the negative feedback cycle is confirmed with Pearson $r = -0.98$;
with a jointly trained backbone, DDCL outperforms its non-differentiable ablation
by $+65\%$ in clustering accuracy and DeepCluster end-to-end by $+122\%$.
\end{abstract}


\end{frontmatter}

\section{Introduction}

Unsupervised representation learning is a central challenge in machine learning:
given a large unlabeled dataset, the goal is to learn feature representations that
capture the intrinsic structure of the data without relying on manual
annotations~\cite{bengio2013}.
Clustering-based approaches offer a natural framework for this task, grouping
similar samples together and using the resulting structure to guide feature
learning~\cite{survey}.

Convolutional neural networks trained on large supervised datasets~\cite{krizhevsky}
have become the dominant paradigm for visual representation learning.
Yet the requirement for labeled data is a fundamental bottleneck: annotations are
expensive, domain-specific, and unavailable in many scientific settings such as
genomics~\cite{lovino} and medical imaging.
This motivates the development of unsupervised methods capable of learning rich
representations from raw data alone.

\subsection{Motivation: the disconnect problem}

A fundamental limitation shared by many deep clustering methods is the disconnect
between representation learning and cluster assignment.
Methods such as DeepCluster~\cite{deepcluster}, DEC~\cite{dec}, and Deep
$k$-Means~\cite{deepkmeans} perform feature extraction and clustering in separate
stages, preventing the two objectives from being jointly optimized through
backpropagation.

In DeepCluster~\cite{deepcluster}, the most prominent example of this paradigm,
a CNN backbone extracts features which are then clustered by $k$-means; the cluster
assignments are converted to discrete pseudo-labels and used to train the backbone
in a subsequent supervised step.
This two-stage procedure has been shown to yield strong unsupervised representations,
but its structural disconnect prevents the backbone from directly optimizing for
cluster quality: it can only predict whatever labels the current $k$-means iteration
happens to produce.

\subsection{Contribution: DDCL}

This paper introduces Deep Dual Competitive Learning (DDCL), the first fully
differentiable end-to-end framework for unsupervised prototype-based deep
clustering.
The core contribution is architectural: the external $k$-means is replaced by
an internal Dual Competitive Layer (DCL)~\cite{dcl}, which feeds the transposed
feature matrix to the competitive layer so that prototypes emerge as native
network outputs rather than internal weight vectors.
This single inversion makes the complete backbone--prototype--assignment pipeline
trainable by backpropagation through a single unified loss.
The theoretical analysis that follows, including the loss decomposition, collapse
analysis, feedback stability, and global Lyapunov theorem, constitutes a rigorous
foundation for understanding why this architectural choice works.

The main contributions are:
\begin{enumerate}
  \item \textbf{The DDCL framework}: a unified differentiable architecture replacing
    external $k$-means with an internal DCL module, in both batch and incremental
    formulations.
  \item \textbf{Loss decomposition theorem}: a proof that
    $\mathcal{L}_q = L_{\mathrm{OLS}} + V$ where $V \ge 0$ is the weighted
    prototype variance under the soft assignment distribution.
  \item \textbf{Gradient analysis}: complete characterization of gradients with
    respect to assignments, prototypes, and backbone parameters, including the role
    of stop-gradient.
  \item \textbf{Collapse analysis}: proof that $\mathcal{L}_q$ makes prototype
    collapse a first-order locally unstable saddle via the implicit separation force
    $\nabla_P V = 2P\Sigma_{q_n}$, where $P \in \mathbb{R}^{d \times k}$ is the
    prototype matrix, $\Sigma_{q_n} = \mathrm{diag}(q_n) - q_nq_n^\top \in \mathbb{R}^{k \times k}$
    is the soft-assignment covariance matrix, and $q_n \in \Delta^{k-1}$ is the
    soft assignment vector of sample $n$ (all formally defined in
    Section~\ref{sec:theory}); this force is absent in $L_{\mathrm{OLS}}$.
  \item \textbf{Feedback stability}: identification and linearized analysis of a
    negative feedback loop between prototype separation and assignment concentration,
    with explicit stability conditions.
  \item \textbf{Global Lyapunov stability}: a rigorous global stability theorem
    for the reduced (frozen-encoder) DDCL flow, establishing that the regularized
    energy $E(P,Q)$, where $P \in \mathbb{R}^{d \times k}$ is the prototype matrix
    and $Q = \{q_n\}_{n=1}^N$ is the collection of soft assignment vectors, is a
    Lyapunov function, all trajectories remain bounded, and every trajectory
    converges to the KKT stationary set.
  \item \textbf{Extensions}: principled extension to CNNs beyond AlexNet, recurrent
    architectures, and Transformers.
\end{enumerate}

It is important to emphasize the scope of the theoretical results.
The decomposition theorem and gradient analysis are exact algebraic identities.
The collapse analysis establishes a first-order local instability of the
collapsed configuration under $\mathcal{L}_q$, not a global convergence
guarantee for the full nonconvex deep system.
The feedback stability analysis is a linearized characterization of the
reduced prototype--assignment dynamics, providing structural insight and a
practical design rule rather than a full nonlinear convergence proof.
The global Lyapunov theorem (Section~\ref{sec:lyapunov}) is a genuine global
result, but applies to the reduced system with frozen encoder features; global
asymptotic stability of the full end-to-end DDCL system remains an open problem
(Section~\ref{sec:limits}).
\textbf{The goal of this paper is not to establish large-scale state-of-the-art
benchmark performance, but to provide the theoretical foundation and controlled
validation of a differentiable prototype-based deep clustering framework.}
The numerical experiments in Section~\ref{sec:experiments} are designed to verify
these structural predictions; large-scale GPU validation on standard visual
benchmarks is identified as the primary direction for future work.

\subsection{Paper organization}

Section~\ref{sec:related} reviews related work.
Section~\ref{sec:background} recalls DeepCluster and DCL.
Section~\ref{sec:ddcl} introduces the DDCL framework.
Section~\ref{sec:theory} develops the theoretical analysis, culminating in the
global Lyapunov stability result.
Section~\ref{sec:extensions} discusses architectural extensions.
Section~\ref{sec:experiments} presents numerical validation.
Section~\ref{sec:discussion} provides a practical discussion.
Section~\ref{sec:conclusion} concludes.
Proofs are collected in the Appendices.

\section{Related Work}
\label{sec:related}

\subsection{Deep clustering}

Deep clustering methods~\cite{survey} combine representation learning with
clustering objectives.
DEC~\cite{dec} minimizes the KL divergence between soft assignments and a sharpened
target distribution, jointly training an autoencoder and cluster centers; its
improved variant IDEC~\cite{idec} preserves local structure through a reconstruction
term.
JULE~\cite{jule} alternates between agglomerative clustering and network updates
in a recurrent framework.
Deep Adaptive Image Clustering~\cite{daic} learns mappings that bring augmented
pairs close in feature space.
Deep $k$-Means~\cite{deepkmeans} embeds $k$-means directly into the network loss
via a differentiable relaxation, but without a dedicated prototype generation module.
Deep Subspace Clustering~\cite{dsc} uses self-expressive layers to capture subspace
structure.

Two recent methods pursue end-to-end differentiable clustering through a different
route.
ClAM~\cite{clam} (ICML 2023) proposes a continuous relaxation of the discrete
assignment problem using Associative Memory dynamics, enabling differentiable
clustering via SGD; however, it operates in the ambient data space and does not
learn latent representations, limiting its applicability to high-dimensional data.
DCAM~\cite{dcam} (NeurIPS 2024) extends this idea to latent space by using
Associative Memory attractor dynamics as an inductive bias inside an autoencoder,
making the encoder, decoder, and cluster prototypes jointly differentiable.
DDCL differs from both in three fundamental respects: (i) prototypes are native
outputs of the backbone graph via the DCL module, not latent attractors of a
recurrent memory system; (ii) no decoder is required, so the framework is
composable with any backbone architecture; (iii) the competitive learning
formulation admits a rigorous algebraic decomposition and Lyapunov analysis
that have no analogue in associative memory methods.

More recent methods exploit self-supervised pretraining.
SCAN~\cite{scan} decouples representation learning from clustering assignment in
two explicit phases: a contrastive pretraining step followed by a neighborhood-based
cluster assignment.
DINO~\cite{dino} shows that Vision Transformer features trained with self-distillation
contain rich semantic structure directly exploitable for clustering without further
training.
CC~\cite{cc} combines instance-level and cluster-level contrastive objectives in a
single end-to-end framework.
Prototypical contrastive learning (PCL,~\cite{pcl}) explicitly maintains cluster
prototypes and uses them as additional contrastive anchors, arriving at a similar
motivation to DDCL but from a contrastive rather than competitive learning
perspective, and without the theoretical decomposition
$\mathcal{L}_q = L_{\mathrm{OLS}} + V$.

The key structural difference between DDCL and all these methods is
summarized as follows.
SwAV~\cite{swav}, which combines contrastive learning with online clustering by
assigning augmented views to shared prototypes, and PCL use prototypes as
contrastive anchors within an augmentation-based objective: prototypes are cluster
centroids maintained outside the backbone graph and updated via an EM or momentum
step, not as differentiable outputs.
SCAN separates pretraining from clustering in two distinct phases with no
joint gradient signal.
DDCL differs in three specific ways: (i) prototypes are \emph{native outputs}
of the backbone graph (DCL module), not external centroids; (ii) the entire
system is optimized by a \emph{single unified loss} $\mathcal{L}_q$ with
continuous gradient flow; (iii) the decomposition
$\mathcal{L}_q = L_{\mathrm{OLS}} + V$ and its consequences (collapse
resistance via $\nabla_P V$, the feedback cycle, and the Lyapunov theorem)
have no analogue in contrastive or EM-based prototype methods.
The present paper focuses on the theoretical analysis of this integration;
empirical comparison with SCAN, DINO, and PCL on large-scale visual
benchmarks is left for future work.

Table~\ref{tab:related} provides a structured comparison of DDCL against
the most relevant existing methods along five dimensions: whether prototypes
are native outputs of the backbone graph, whether the framework requires a
reconstruction decoder, whether training uses a single differentiable loss
with no alternating steps, whether data augmentation is needed, and whether
a rigorous theoretical analysis of the loss geometry accompanies the method.
The table covers the full spectrum from classical deep clustering (DEC,
DeepCluster, Deep $k$-Means) through modern self-supervised methods (SwAV,
PCL) to the two most closely related differentiable clustering approaches
(ClAM, DCAM).

\begin{table}[!htbp]
\centering
\caption{Structural comparison of deep clustering methods ($\checkmark$~=~satisfied,
\texttimes~=~not satisfied).
$\dagger$: ambient data space only, no latent representation.
$\ddagger$: requires encoder + decoder.}
\label{tab:related}
\begin{tabular}{lccccc}
\toprule
Method & \parbox{1.5cm}{\centering Native\\protos} &
         \parbox{1.5cm}{\centering No\\decoder} &
         \parbox{1.5cm}{\centering Unified\\loss} &
         \parbox{1.8cm}{\centering No\\augmentation} &
         Theory \\
\midrule
DEC~\cite{dec}           & \texttimes & \texttimes & \checkmark & \checkmark & \texttimes \\
DeepCluster~\cite{deepcluster} & \texttimes & \checkmark & \texttimes & \checkmark & \texttimes \\
Deep $k$-Means~\cite{deepkmeans} & \texttimes & \checkmark & \checkmark & \checkmark & \texttimes \\
SwAV~\cite{swav}         & \texttimes & \checkmark & \checkmark & \texttimes & \texttimes \\
PCL~\cite{pcl}           & \texttimes & \checkmark & \checkmark & \texttimes & \texttimes \\
ClAM$^\dagger$~\cite{clam} & \texttimes & \checkmark & \checkmark & \checkmark & \texttimes \\
DCAM$^\ddagger$~\cite{dcam} & \texttimes & \texttimes & \checkmark & \checkmark & \texttimes \\
\midrule
\textbf{DDCL (proposed)} & \checkmark & \checkmark & \checkmark & \checkmark & \checkmark \\
\bottomrule
\end{tabular}
\end{table}

The key observation from Table~\ref{tab:related} is that DDCL is the only
method that simultaneously satisfies all five criteria.
The columns denote: \emph{Native protos} --- prototypes are differentiable outputs
of the backbone graph, not external centroids; \emph{No decoder} --- no reconstruction
decoder is required; \emph{Unified loss} --- a single differentiable objective with
continuous gradient flow; \emph{No augmentation} --- no data augmentation or
contrastive views; \emph{Theory} --- rigorous analysis of loss geometry and dynamics.
The most discriminating column is \emph{Native protos}: no prior method
generates prototypes as differentiable outputs of the backbone computational
graph.
DEC and DeepCluster maintain cluster centres as external parameters updated
outside the gradient flow.
SwAV and PCL update prototypes via Sinkhorn-Knopp or EM steps that are
decoupled from the backbone.
ClAM and DCAM achieve differentiable clustering but through associative
memory dynamics rather than competitive learning: ClAM operates only in the
ambient data space without representation learning, while DCAM requires a
full autoencoder and uses memory attractors as prototypes rather than native
network outputs.
The \emph{Theory} column reflects that, unlike all of these methods, DDCL
comes with a rigorous analysis: an exact loss decomposition, a characterisation
of collapse resistance, a linearised feedback stability result, and a global
Lyapunov theorem for the reduced system.

Self-supervised methods~\cite{doersch,noroozi} avoid explicit clustering by
designing pretext tasks.
More recently, contrastive methods such as SimCLR~\cite{simclr} and
MoCo~\cite{moco} learn representations by contrasting augmented views of the same
image.
BYOL~\cite{byol} and SimSiam~\cite{simsiam} operate without negative pairs.
SwAV~\cite{swav} bridges contrastive and clustering objectives (discussed in
Section~\ref{sec:related} in relation to DDCL).

\subsection{Competitive learning}

Competitive learning~\cite{rumelhart} is a biologically inspired paradigm where
neurons compete to represent input patterns.
The Self-Organizing Map (SOM) \cite{kohonen} organizes prototypes on a fixed
low-dimensional grid, combining competition with topology preservation.
Growing Neural Gas (GNG)~\cite{fritzke} relaxes the fixed topology by incrementally
adding units where the quantization error is highest.
Both methods learn prototype positions through Competitive Hebbian Learning
(CHL)~\cite{martinetz}, which encodes the first- and second-winner structure in a
topology graph.
The GH-EXIN network~\cite{ghexin} extends these ideas to hierarchical clustering
by nesting competitive layers at multiple granularity levels.

A fundamental limitation shared by all these methods is that prototypes are weight
vectors internal to a single layer: they are not accessible as differentiable outputs
and cannot be backpropagated through from a downstream loss.
The DCL~\cite{dcl} resolves this by transposing the input matrix before the
competitive layer, so that prototypes appear as explicit, differentiable columns
of the output matrix (formally defined in Section~\ref{sec:dcl}).
This architectural inversion makes prototype generation fully differentiable and
composable with any downstream objective.
It also enables the gradient boundedness property (DCL gradients live in the
$n$-dimensional data subspace regardless of $d$, the feature dimension),
which is central to the high-dimensional robustness of DDCL.
The regression interpretation of DCL, connecting competitive learning to OLS and
its variants (TLS, weighted LS), is developed in~\cite{dcl} and further grounded
in the neural data-fitting framework of~\cite{exin}.

\section{Background}
\label{sec:background}

\subsection{DeepCluster}

Let $\{x_n\}_{n=1}^N$ be $N$ unlabeled input samples.
A backbone CNN $f_\theta$ with parameters $\theta$ extracts features
$z_n = f_\theta(x_n) \in \mathbb{R}^d$, collected in the feature matrix
$F_\theta = [z_1,\dots,z_N] \in \mathbb{R}^{d \times N}$,
where $d$ is the feature dimension.
Throughout the paper, $k$ denotes the number of prototypes (clusters), specified
\emph{a priori}.

DeepCluster~\cite{deepcluster} proceeds as follows:
\begin{enumerate}
  \item \textbf{Feature preprocessing}: PCA to $D \ll d$ dimensions,
    $\ell_2$-normalization, whitening.
  \item \textbf{External clustering}: run $k$-means on $\{z_n\}$ to obtain
    centroids $C \in \mathbb{R}^{d \times k}$ and hard assignments
    $y_n \in \{0,1\}^k$, $y_n^\top \mathbf{1}_k = 1$.
  \item \textbf{Pseudo-label training}: train $f_\theta$ and a linear classifier
    head $g_\phi$ (weights $\phi$) to predict the pseudo-labels.
    The classifier output is:
    \begin{equation}
      \hat{s}_n = \mathrm{softmax}(g_\phi(z_n)) \in \mathbb{R}^k,
    \end{equation}
    where $\hat{s}_{n,j}$ is the predicted probability of sample $n$ belonging to
    cluster $j$.
    The cross-entropy loss is:
    $L_{\mathrm{DC}} = -\frac{1}{N}\sum_n \log \hat{s}_{n,y_n}$,
    where $y_n \in \{1,\dots,k\}$ is the integer cluster index assigned by $k$-means
    (i.e., the component of $y_n$ that equals 1 in the one-hot encoding).
  \item Repeat until convergence.
\end{enumerate}

The $k$-means step solves:
\begin{equation}
  \min_{C,\{y_n\}} \frac{1}{N}\sum_n \|z_n - Cy_n\|^2
  \quad \text{s.t.}\quad y_n \in \{0,1\}^k,\; y_n^\top \mathbf{1}_k = 1.
\end{equation}
DeepCluster includes two mechanisms to prevent trivial solutions: empty cluster
reassignment and uniform cluster sampling~\cite{deepcluster}.

The fundamental limitation is structural: $k$-means is external to the computational
graph.
Gradients do not flow through the assignment step, so the backbone cannot directly
optimize for cluster quality.

\subsection{The Dual Competitive Layer (DCL)}
\label{sec:dcl}

The DCL was introduced in~\cite{dcl} as a gradient-based competitive learning
module that generates prototypes as differentiable network outputs by operating
on the transposed feature matrix.
In what follows, $X \in \mathbb{R}^{n \times d}$ denotes a generic input matrix
of $n$ samples (which in the DDCL context corresponds to the current batch,
$n \le N$); $d$ is the feature dimension and $k$ the number of prototypes, as
defined in Section~\ref{sec:background}.
The construction and the two properties most relevant to DDCL are recalled below.

\paragraph{Vanilla Competitive Layer (VCL)}
The VCL is the classical competitive learning unit~\cite{rumelhart,kohonen}
as it appears in the neural network literature.
Given an input matrix $X \in \mathbb{R}^{n \times d}$ ($n$ samples, each with
$d$ features, arranged as rows),
a VCL with $k$ neurons (where $k$ is specified \emph{a priori} and may not
equal the true number of clusters in the data) has weight matrix
$W_1 \in \mathbb{R}^{k \times d}$, where
each row $w_j^\top \in \mathbb{R}^d$ is a prototype vector.
For each sample $x_i \in \mathbb{R}^d$, the winning neuron is $j^* = \arg\min_j \|x_i - w_j\|^2$, and the
loss is the quantization error:
\begin{equation}
  Q = \sum_{i=1}^n \min_j \|x_i - w_j\|^2.
\end{equation}
Prototypes are stored internally as weight vectors and are not directly accessible as network
outputs; they cannot be connected to a downstream differentiable loss.

\paragraph{Dual Competitive Layer (DCL)}
The key insight of~\cite{dcl} is to feed the transposed matrix
$X^\top \in \mathbb{R}^{d \times n}$ to the competitive layer.
With $n$ input neurons and $k$ output neurons, the DCL weight matrix is
$W_2 \in \mathbb{R}^{n \times k}$, whose $j$-th column $w_{2,j} \in \mathbb{R}^n$
contains the weights connecting all $n$ input neurons to the $j$-th output neuron.
The output layer activation matrix is:
\begin{equation}
  P = Y_2 = X^\top W_2 \in \mathbb{R}^{d \times k},
  \label{eq:dcl_output}
\end{equation}
where $Y_2$ denotes the output activation matrix of the DCL (i.e., the result
of multiplying the transposed input $X^\top$ by the weight matrix $W_2$),
whose $j$-th column $p_j = X^\top w_{2,j} \in \mathbb{R}^d$ is the $j$-th prototype,
expressed as a linear combination of the input data columns weighted by $w_{2,j}$.
Prototypes are now explicit network outputs: gradients from any downstream loss
flow through $P$ back into $W_2$ and, via $X^\top$, into the backbone that produced
$X$.

The DCL loss augments the quantization error with a topology regularization term:
\begin{equation}
  L_{\mathrm{DCL}} = Q + \lambda \|E\|^2,
\end{equation}
where $E \in \mathbb{R}^{k \times k}$ is the Competitive Hebbian Learning (CHL)
adjacency matrix~\cite{martinetz}, encoding the first- and second-winner
co-activation structure across the training samples.
Minimizing $\|E\|^2$ penalizes dense connectivity between prototypes, enforcing a
\emph{sparse topology}: only pairs of prototypes that are mutually nearest neighbours
are linked, producing a graph whose connected components tend to reflect the true cluster
structure of the data.
This is particularly valuable in practice, since $k$ is set \emph{a priori} and may
overestimate the actual number of clusters; the sparse topology provides a visual and
algorithmic means of identifying the true cluster count without rerunning the model.

\paragraph{Duality and equivalence}
Reference~\cite{dcl} proves (Theorems~2.1--2.4) that VCL and DCL produce the same
prototype geometry under the assumption that input samples are uncorrelated with
unit variance, a condition practically achieved by batch normalization.
The two architectures are therefore equivalent in representation power while DCL
is strictly richer in connectivity: the same competitive dynamics, but with
prototypes promoted to explicit, differentiable output tensors that participate
fully in the computation graph.

\paragraph{Gradient subspace property}
Theorems~3.4--3.5 of~\cite{dcl} establish that the DCL gradient
$\nabla_{W_2} L_{\mathrm{DCL}}$ lies in the column space of $X^\top$, an
$n$-dimensional subspace of $\mathbb{R}^d$.
This result has a deep practical consequence that is central to DDCL's behavior
in high-dimensional settings.
In the ambient space $\mathbb{R}^d$, most directions are uninformative: when
$d \gg n$, only $n$ linearly independent directions are spanned by the data, and
the remaining $d - n$ directions contain only noise.
The gradient subspace property means that DCL updates are \emph{automatically
confined to the $n$-dimensional data subspace}, a geometric tunnel aligned
with the actual information content of the batch, and completely immune
to the noise that contaminates the $d - n$ orthogonal directions.
This is not a regularization imposed from outside but a structural property of
the dual architecture: because the prototype $p_j = X^\top w_{2,j}$ is formed
by taking a linear combination of the $n$ data vectors, its gradient can never
escape their span.
The practical effect is striking: as $d$ grows beyond $n$, methods that operate
in the full ambient space (including standard $k$-means and DeepCluster) see
their signal-to-noise ratio collapse, while DDCL maintains meaningful gradient
updates through this subspace tunnel.
Block~3 (Section~\ref{sec:block3}) validates this property experimentally across
six dimensionalities from $d = 10$ to $d = 5000$, showing that
DDCL($\mathcal{L}_q$) degrades gracefully at $d > n$ while all ambient-space
methods converge to chance level.

\section{Deep Dual Competitive Learning (DDCL)}
\label{sec:ddcl}

\subsection{Conceptual shift}

DDCL replaces the external $k$-means in DeepCluster with a DCL block.
The transpose of the feature matrix $F_\theta \in \mathbb{R}^{d \times N}$ is passed
to DCL:
\begin{equation}
  F_\theta^\top \xrightarrow{\mathrm{DCL}} P = [p_1,\dots,p_k] \in \mathbb{R}^{d \times k}.
\end{equation}
DCL outputs the prototype matrix $P$ as network outputs, differentiable end-to-end.
The full pipeline is:
\begin{equation}
  x_n \to f_\theta \to z_n,\quad
  F_\theta^\top \to \mathrm{DCL} \to P,\quad
  (z_n, P) \to q_n \to \mathcal{L}.
\end{equation}
Fig.~\ref{fig:pipeline} illustrates the batch DDCL pipeline with CNN backbone.

\begin{figure}[!htbp]
  \centering
  \includegraphics[width=\linewidth]{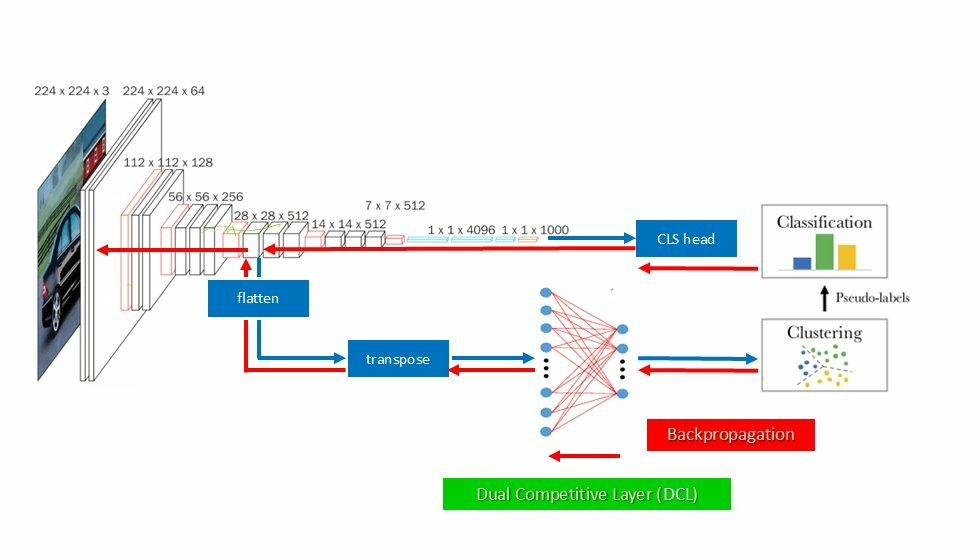}
  \caption{Batch DDCL pipeline with CNN backbone (see Section~\ref{sec:ddcl}).
    Blue: forward pass; red: backpropagation.}
  \label{fig:pipeline}
\end{figure}

\subsection{Batch DDCL with cross-entropy}

The simplest DDCL variant retains the pseudo-label logic of DeepCluster.
From $P$, define soft assignments with temperature $T > 0$:
\begin{equation}
  q_{nj} = \frac{\exp(-\|z_n - p_j\|^2/T)}{\sum_\ell \exp(-\|z_n - p_\ell\|^2/T)},
  \label{eq:softassign}
\end{equation}
and train a linear classifier head with weight matrix $W_c \in \mathbb{R}^{k \times d}$
and bias $b \in \mathbb{R}^k$, with output probability vector
\begin{equation}
  \hat{s}_n = \mathrm{softmax}(W_c z_n + b) \in \mathbb{R}^k,
\end{equation}
via the soft cross-entropy loss:
\begin{equation}
  L_{\mathrm{softCE}} = -\frac{1}{N}\sum_n\sum_j q_{nj}\log\hat{s}_{nj}.
\end{equation}
Here $\hat{s}_{nj}$ is the predicted probability that sample $n$ belongs to
cluster $j$, and $q_{nj}$ provides the soft target in place of the hard
pseudo-labels used by DeepCluster.

\subsection{Unified batch DDCL}

In the unified formulation, the separate classifier head is eliminated.
The soft quantization loss is:
\begin{equation}
  \mathcal{L}_q = \frac{1}{N}\sum_{n=1}^N\sum_{j=1}^k q_{nj}\|z_n - p_j\|^2,
  \label{eq:Lq}
\end{equation}
with $q_{nj}$ from~\eqref{eq:softassign} denoting the $j$-th component
of the soft assignment vector $q_n \in \Delta^{k-1}$.
This single loss jointly trains the backbone $f_\theta$ and the DCL module,
requiring no Lloyd iterations, no hard pseudo-labels, and no classifier head.

To prevent trivial solutions, the total loss augments $\mathcal{L}_q$:
\begin{equation}
  L_{\mathrm{tot}} = \mathcal{L}_q + \beta L_{\mathrm{bal}}
                   + \gamma L_{\mathrm{ent}}
                   + \eta L_{\mathrm{sep}}
                   + \tfrac{\lambda}{2}\|P\|_F^2,
  \label{eq:Ltot}
\end{equation}
where $\beta, \gamma, \eta, \lambda \ge 0$ are non-negative weights,
$\bar{q} = \frac{1}{N}\sum_n q_n \in \Delta^{k-1}$ is the mean soft assignment
across the batch, and:
\begin{align}
  L_{\mathrm{bal}}  &= \mathrm{KL}\!\left(\bar{q}\,\Big\|\,\tfrac{1}{k}\mathbf{1}_k\right),
    \label{eq:Lbal}\\
  L_{\mathrm{ent}}  &= \frac{1}{N}\sum_n H(q_n),\quad
    H(q_n) = -\sum_j q_{nj}\log q_{nj},
    \label{eq:Lent}\\
  L_{\mathrm{sep}}  &= -\sum_{i<j}\|p_i - p_j\|^2.
    \label{eq:Lsep}
\end{align}
Each term addresses a specific failure mode.
$L_{\mathrm{bal}}$ is the KL divergence between the mean assignment $\bar{q}$ and
the uniform distribution $\tfrac{1}{k}\mathbf{1}_k$: minimizing it encourages all
$k$ clusters to be used roughly equally, preventing the model from collapsing to a
solution where most samples are assigned to a single prototype.
$L_{\mathrm{ent}}$ is the mean Shannon entropy $H(q_n) = -\sum_j q_{nj}\log q_{nj}$
of the per-sample soft assignments: maximizing it (i.e., penalizing its negative)
keeps assignments soft and spread, avoiding premature hard convergence before
meaningful prototype positions have been established.
$L_{\mathrm{sep}} = -\sum_{i<j}\|p_i - p_j\|^2$ directly penalizes pairs of
prototypes that are too close, providing an explicit push-apart force that
complements the implicit separation already provided by $\nabla_P V$
(Section~\ref{sec:collapse}).
The quadratic term $\tfrac{\lambda}{2}\|P\|_F^2$ prevents prototypes from drifting
to infinity under the repulsion of $L_{\mathrm{sep}}$; its role in ensuring
coercivity and hence the global Lyapunov result is discussed in
Remark~\ref{rem:coercivity}.

\begin{remark}[Role of the quadratic term $(\lambda/2)\|P\|_F^2$]
\label{rem:coercivity}
The explicit separation term $L_{\mathrm{sep}}$ is unbounded below as
$\|p_i - p_j\| \to \infty$, so it cannot by itself guarantee that prototype
trajectories remain bounded.
The quadratic regularizer $(\lambda/2)\|P\|_F^2$ ($\lambda > 0$) introduces
a restoring force towards the origin that, together with $L_{\mathrm{sep}}$,
confines the prototypes to a bounded region: $L_{\mathrm{sep}}$ pushes prototypes
apart, while $(\lambda/2)\|P\|_F^2$ prevents them from drifting to infinity.
This coercivity is essential for the global Lyapunov result of
Section~\ref{sec:lyapunov}.
In practice $\lambda$ is chosen small (e.g.\ $10^{-4}$--$10^{-3}$) so that its
effect on prototype geometry is negligible.
\end{remark}

\subsection{Constrained OLS interpretation}

The name $L_{\mathrm{OLS}}$ reflects the classical \emph{Ordinary Least Squares}
problem from linear regression.
Given a feature vector $z_n \in \mathbb{R}^d$ and a prototype matrix
$P = [p_1,\dots,p_k] \in \mathbb{R}^{d \times k}$, the goal is to find
an assignment vector $q_n \in \mathbb{R}^k$ such that the linear combination
$Pq_n$ approximates $z_n$ as closely as possible.
This is precisely the OLS problem $\min_{q_n} \|z_n - Pq_n\|^2$, which in
unconstrained form has the closed-form solution $q_n = (P^\top P)^{-1}P^\top z_n$.
The constraint $q_n \in \Delta^{k-1}$, i.e., the entries of $q_n$ are
non-negative and sum to one, turns it into a \emph{simplex-constrained} OLS
problem, which no longer has a closed form but admits efficient projected-gradient
solvers.
Relaxing the hard one-hot constraint in~\eqref{eq:Lq} to the probability simplex
$\Delta^{k-1} = \{q \ge 0 : \mathbf{1}^\top q = 1\}$ gives the
\emph{simplex-constrained OLS loss}:
\begin{equation}
  L_{\mathrm{OLS}} = \|z_n - Pq_n\|^2 = \|z_n - \bar{p}_n\|^2,\quad
  \bar{p}_n = Pq_n.
  \label{eq:LOLS}
\end{equation}
Hard DDCL ($T \to 0$) recovers a discrete one-hot OLS problem equivalent to
$k$-means.
Soft DDCL is a continuous simplex-constrained OLS problem.

\begin{remark}[Beyond OLS: noise assumptions and regression variants]
OLS assumes that observation errors lie entirely in the feature space $z_n$
(errors-in-response model).
When both observations $z_n$ and prototype coordinates $P$ are subject to noise,
Total Least Squares (TLS) provides a more appropriate formulation, minimizing
orthogonal distances to the prototype subspace rather than vertical residuals.
More generally, the choice of regression criterion (OLS, TLS, weighted LS, or
robust variants) should be guided by the noise structure of the representation
space; a systematic treatment of these regression problems and their neural
implementations is given in~\cite{exin}.
\end{remark}

\subsection{Incremental DDCL}
\label{sec:incremental}

In the incremental formulation, prototype rows are generated sequentially: the DCL
produces one row $r^{(t)\top} \in \mathbb{R}^k$ per step.
After $t$ steps, the partial prototype matrix is $P_{1:t} \in \mathbb{R}^{t \times k}$.
Defining truncated features $z_{n,1:t} \in \mathbb{R}^t$, the incremental
assignment update is:
\begin{equation}
  q_n^{(t)} = \arg\min_{q \in \Delta^{k-1}} \|z_{n,1:t} - P_{1:t}q\|^2.
\end{equation}
This is solved online via Widrow--Hoff with simplex projection:
\begin{align}
  e_n^{(t)} &= z_{nt} - r^{(t)\top} q_n^{(t-1)},\\
  \tilde{q}_n^{(t)} &= q_n^{(t-1)} + \mu\, e_n^{(t)} r^{(t)},\\
  q_n^{(t)} &= \Pi_{\Delta}^{k-1}\!\left(\tilde{q}_n^{(t)}\right),
\end{align}
where $e_n^{(t)} \in \mathbb{R}$ is the scalar prediction error at step $t$,
$\mu > 0$ is the Widrow--Hoff step size,
and $\Pi_{\Delta}^{k-1}$ is the Euclidean projection onto $\Delta^{k-1}$,
computed via the efficient algorithm of~\cite{dcl}.

\paragraph{Simplex projection}
The Euclidean projection $\Pi_{\Delta}^{k-1}(v)$ maps any vector
$v \in \mathbb{R}^k$ to the nearest point in the probability simplex
$\Delta^{k-1} = \{q \in \mathbb{R}^k : q \ge 0,\, \mathbf{1}^\top q = 1\}$.
It can be computed exactly in $O(k \log k)$ time via the following
sorting-based algorithm~\cite{dcl}:
\begin{enumerate}
  \item Sort $v$ in descending order: $v_{(1)} \ge v_{(2)} \ge \cdots \ge v_{(k)}$.
  \item Find the largest $\rho \in \{1,\dots,k\}$ such that
    $v_{(\rho)} - \frac{1}{\rho}\!\left(\sum_{j=1}^\rho v_{(j)} - 1\right) > 0$.
  \item Set the threshold $\theta^* = \frac{1}{\rho}\!\left(\sum_{j=1}^\rho v_{(j)} - 1\right)$.
  \item Return $\Pi_{\Delta}^{k-1}(v) = \max(v - \theta^*\mathbf{1}_k,\, 0)$.
\end{enumerate}
This operation replaces the unconstrained update $\tilde{q}_n^{(t)}$ with the
nearest feasible assignment vector, ensuring $q_n^{(t)} \in \Delta^{k-1}$
at every step with no iterative inner loop.
A complete pseudocode of the incremental loop is given in~\ref{app:alg:incr}.
The total incremental objective is:
\begin{equation}
  L_{\mathrm{inc}} = \sum_{t=1}^d \alpha_t L^{(t)},\quad
  L^{(t)} = \frac{1}{N}\sum_n \|z_{n,1:t} - P_{1:t}q_n^{(t)}\|^2,
\end{equation}
with non-negative weights $\alpha_t$ typically increasing with $t$.

\section{Theoretical Analysis}
\label{sec:theory}

The core theoretical results are developed below.
Throughout this section a single sample $z_n \in \mathbb{R}^d$, prototypes
$P = [p_1,\dots,p_k] \in \mathbb{R}^{d \times k}$, and a soft assignment
$q_n \in \Delta^{k-1}$ are fixed.
The notation is: $\bar{p}_n = Pq_n$ (prototype mixture),
$d_n = (\|z_n - p_1\|^2, \dots, \|z_n - p_k\|^2)^\top \in \mathbb{R}^k$
(distance vector).
All proofs are in the Appendices.

\subsection{Loss decomposition}

\begin{theorem}[Loss decomposition]
\label{thm:decomp}
For any $z_n \in \mathbb{R}^d$, $P \in \mathbb{R}^{d \times k}$,
$q_n \in \Delta^{k-1}$:
\begin{equation}
  \mathcal{L}_q = L_{\mathrm{OLS}} + V(P, q_n),
  \label{eq:decomp}
\end{equation}
where the variance term
\begin{equation}
  V(P, q_n) = \sum_j q_{nj}\|p_j - \bar{p}_n\|^2 \ge 0
  \label{eq:V}
\end{equation}
is the weighted variance of the prototypes under $q_n$.
Equality holds iff $q_n$ is a vertex of $\Delta^{k-1}$ or all active prototypes
coincide.
\end{theorem}

\begin{corollary}[Upper bound]
$0 \le V(P,q_n) \le \tfrac{1}{4}\mathrm{diam}(P)^2$,
where $\mathrm{diam}(P) = \max_{i \ne j}\|p_i - p_j\|$ is the prototype diameter.
\end{corollary}

The decomposition~\eqref{eq:decomp} reveals that $\mathcal{L}_q$ always
overestimates $L_{\mathrm{OLS}}$ by the prototype variance.
The gap vanishes only in the hard-assignment limit and grows with prototype
dispersion.

\subsection{Temperature and entropic structure}

Since $\mathcal{L}_q = d_n^\top q_n$ is linear in $q_n$, its unconstrained
minimizer over $\Delta^{k-1}$ is always a vertex (hard assignment, equivalent to
$k$-means).
Soft assignments arise solely from entropic regularization:

\begin{proposition}[Entropic regularization]
\label{prop:entropic}
The soft assignments in~\eqref{eq:softassign}, $q_{nj}(T) \propto
\exp(-d_{nj}/T)$, are the unique minimizer of:
\begin{equation}
  \min_{q_n \in \Delta^{k-1}} \left\{ d_n^\top q_n - T \cdot H(q_n) \right\},
\end{equation}
with $\lim_{T \to 0} q_n(T) = e_{j^*}$ ($k$-means hard assignment, where
$e_{j^*} \in \{0,1\}^k$ is the standard basis vector for the nearest prototype
$j^* = \arg\min_j d_{nj}$) and
$\lim_{T \to \infty} q_n(T) = \tfrac{1}{k}\mathbf{1}_k$ (uniform).
\end{proposition}

\begin{corollary}[Monotonicity of $V$]
$V(P, q_n(T))$ is monotone non-decreasing in $T$, with $V|_{T \to 0} = 0$ and
$V|_{T \to \infty} = \mathrm{Var}(P) \triangleq \tfrac{1}{k}\sum_j\|p_j\|^2
- \|\tfrac{1}{k}\sum_j p_j\|^2$ (the unweighted prototype variance).
\end{corollary}

Temperature $T$ is therefore the master parameter of DDCL: it controls the
softness of assignments, the magnitude of $V$, the backbone learning signal,
and the implicit separation force (Section~\ref{sec:collapse}).

\subsection{Gradient analysis}

Define the assignment covariance matrix:
\begin{equation}
  \Sigma_{q_n} = \mathrm{diag}(q_n) - q_n q_n^\top \in \mathbb{R}^{k \times k}.
  \label{eq:Sigma}
\end{equation}

\begin{proposition}[Gradients w.r.t.\ $q_n$]
\label{prop:grad_q}
\begin{align}
  \nabla_{q_n} L_{\mathrm{OLS}} &= 2P^\top(\bar{p}_n - z_n),\\
  \nabla_{q_n} \mathcal{L}_q   &= d_n \quad\text{(constant in $q_n$)},\\
  \nabla_{q_n} V               &= (\|p_j\|^2)_j - 2P^\top\bar{p}_n.
\end{align}
\end{proposition}

\begin{proposition}[Gradients w.r.t.\ $P$]
\label{prop:grad_P}
\begin{align}
  \nabla_P L_{\mathrm{OLS}} &= 2(\bar{p}_n - z_n)q_n^\top,\\
  \nabla_P \mathcal{L}_q   &= 2(P - z_n\mathbf{1}_k^\top)\mathrm{diag}(q_n),\\
  \nabla_P V               &= 2P\Sigma_{q_n}.
\end{align}
\end{proposition}

\begin{proposition}[Gradients w.r.t.\ $z_n$]
\label{prop:grad_z}
Under stop-gradient on $q_n$:
\begin{equation}
  \nabla_{z_n} \mathcal{L}_q\big|_{\mathrm{sg}}
  = \nabla_{z_n} L_{\mathrm{OLS}}\big|_{\mathrm{sg}}
  = 2(z_n - \bar{p}_n).
\end{equation}
Under full gradient, both losses acquire a correction $\mathcal{O}(T^{-1})$
proportional to $\Sigma_{q_n}$, and
$\nabla_{z_n}\mathcal{L}_q - \nabla_{z_n}L_{\mathrm{OLS}} = \nabla_{z_n}V$.
\end{proposition}

\begin{corollary}[Stop-gradient as design choice]
With stop-gradient on $q_n$, $\mathcal{L}_q$ and $L_{\mathrm{OLS}}$ give identical
backbone updates.
The difference between the two losses is felt exclusively in the DCL module
gradients $\nabla_P$ and the assignment gradients $\nabla_{q_n}$.
Without stop-gradient, both losses enrich the backbone signal by a term
$\propto T^{-1}\Sigma_{q_n}$, coupling backbone and assignment dynamics.
\end{corollary}

\noindent
\textit{Properties of $\Sigma_{q_n}$}: $\Sigma_{q_n} \succeq 0$;
$\mathrm{rank}(\Sigma_{q_n}) \le k-1$ with $q_n \in \ker(\Sigma_{q_n})$;
$\mathrm{tr}(\Sigma_{q_n}) = 1 - \|q_n\|^2 \in [0,(k-1)/k]$;
$\Sigma_{q_n} \to 0$ as $T \to 0$;
$\Sigma_{q_n} \to \tfrac{1}{k}I_k - \tfrac{1}{k^2}\mathbf{1}\mathbf{1}^\top$
as $T \to \infty$.

\subsection{Prototype collapse}
\label{sec:collapse}

Prototype collapse (multiple prototypes converging to the same point) is one
of the most dangerous failure modes in deep clustering.

\begin{proposition}[$L_{\mathrm{OLS}}$: collapse is a stable fixed point]
If $p_i = p_j = p^*$, then $\partial L_{\mathrm{OLS}}/\partial p_i$ and
$\partial L_{\mathrm{OLS}}/\partial p_j$ point in the same direction.
No differential force acts to separate them; collapse is a stable attractor of
$L_{\mathrm{OLS}}$.
\end{proposition}

\begin{proposition}[$\mathcal{L}_q$: collapse is an unstable saddle]
Consider a perturbation $p_i = p^* + \varepsilon u$, $p_j = p^* - \varepsilon u$,
$\|u\| = 1$.
Under gradient descent on $\mathcal{L}_q$, the net force along $u$ separating
$p_i$ from $p_j$ is:
\begin{equation}
  F_{\mathrm{sep}} = 2\varepsilon(q_{ni} + q_{nj}) > 0,
\end{equation}
absent under $L_{\mathrm{OLS}}$ at first order.
Hence collapse is an unstable saddle of $\mathcal{L}_q$.
\end{proposition}

The source of this implicit separation is $\nabla_P V = 2P\Sigma_{q_n}$.
To see why, recall from the decomposition $\mathcal{L}_q = L_{\mathrm{OLS}} + V$
that minimizing $\mathcal{L}_q$ simultaneously minimizes the reconstruction error
\emph{and} the prototype variance $V$.
The gradient $\nabla_P V = 2P\Sigma_{q_n}$ points in the direction that
\emph{increases} $V$, so the gradient \emph{descent} step on $\mathcal{L}_q$
pushes the prototypes in the direction that \emph{decreases} $V$ --- that is,
it increases the spread of the prototype distribution.
Since the assignment covariance $\Sigma_{q_n} = \mathrm{diag}(q_n) - q_nq_n^\top$
is positive semidefinite with eigenvalues that grow when assignments are more
uniform, the force $P\Sigma_{q_n}$ is strongest when the prototypes are close
and the assignments are undecided --- precisely the regime of incipient collapse.
This is why $\nabla_P V$ acts as a self-regulating separation force: it is
automatically largest where it is most needed.

\begin{proposition}[Separation intensity vs.\ $T$]
$\|\nabla_P V\|_F = 2\|P\Sigma_{q_n}\|_F$ is monotone non-decreasing in $T$,
vanishing as $T \to 0$.
At $T \to \infty$:
\begin{equation}
  \|\nabla_P V\|_F \to 2\left\|P\!\left(\tfrac{1}{k}I_k
  - \tfrac{1}{k^2}\mathbf{1}\mathbf{1}^\top\right)\right\|_F.
\end{equation}
\end{proposition}

\begin{corollary}[Role of $L_{\mathrm{sep}}$]
The explicit term $L_{\mathrm{sep}}$ in~\eqref{eq:Lsep} acts as a backup mechanism,
most necessary when $T$ is small and the implicit force $\nabla_P V$ is weak.
The weight $\eta$ should be chosen inversely proportional to
$T \cdot \mathrm{tr}(\Sigma_{q_n}(T))$.
\end{corollary}

\subsection{Negative feedback and stability}
\label{sec:feedback}

To analyze the self-regulating dynamics of DDCL, three scalar
quantities that capture the essential state of the system at each training step.
Intuitively, the training process involves two competing tendencies: prototypes
want to separate (to cover the data better), while assignments want to sharpen
(to commit to a single prototype per sample).
The three quantities below make this competition precise.

The feedback quantities are defined as:
\begin{align}
  \mathcal{S}(P)     &= \tfrac{1}{\binom{k}{2}}\sum_{i<j}\|p_i - p_j\|^2
    \quad\text{(separation)},\\
  \mathcal{K}(q_n)   &= \|q_n\|^2
    \quad\text{(concentration, $\in [1/k,1]$)},\\
  \mathcal{I}(\Sigma_{q_n}) &= \mathrm{tr}(\Sigma_{q_n}) = 1 - \mathcal{K}
    \quad\text{(separation force)}.
\end{align}
$\mathcal{S}(P)$ measures how spread out the prototypes are on average: it is
zero when all prototypes coincide (total collapse) and grows as they move apart.
$\mathcal{K}(q_n) = \|q_n\|^2$ measures how concentrated the soft assignment of
sample $n$ is: it equals $1$ when $q_n$ is a one-hot vector (hard assignment to
a single prototype) and equals $1/k$ when $q_n = \tfrac{1}{k}\mathbf{1}_k$
(maximally uncertain, uniform over all prototypes).
$\mathcal{I}(\Sigma_{q_n}) = 1 - \mathcal{K}$ is the complementary measure of
assignment uncertainty, and --- as established by $\nabla_P V = 2P\Sigma_{q_n}$
above --- it directly controls the magnitude of the implicit separation force.

The cycle reads:
\begin{equation}
  \mathcal{S} \uparrow \;\Rightarrow\;
  \mathcal{K} \uparrow \;\Rightarrow\;
  \mathcal{I} \downarrow \;\Rightarrow\;
  \|\nabla_P V\|_F \downarrow \;\Rightarrow\;
  \mathcal{S} \downarrow.
\end{equation}
This chain of implications describes a self-correcting loop, readable as follows.
When prototypes are well separated ($\mathcal{S}$ large), each sample clearly
belongs to one prototype: assignments sharpen ($\mathcal{K}$ increases).
Sharp assignments mean $\Sigma_{q_n}$ shrinks toward zero ($\mathcal{I}$ falls),
which weakens the implicit separation force $\|\nabla_P V\|_F$.
With less force pushing prototypes apart, separation $\mathcal{S}$ naturally
decreases --- closing the loop.
The same logic runs in reverse when prototypes are too close: uniform assignments
maximize $\mathcal{I}$, which maximizes the separation force and pushes prototypes
apart.
The practical consequence is that the system self-corrects in both directions:
it resists prototype collapse \emph{and} resists excessive dispersion, without
any external intervention.

\begin{theorem}[Negative feedback]
\label{thm:feedback}
The cycle is negative: it resists both collapse ($\mathcal{S} \to 0$) and
excessive dispersion ($\mathcal{S} \to \infty$), and there exists at least one
equilibrium $(\mathcal{S}^*, \mathcal{K}^*)$.
\end{theorem}

Linearizing around the equilibrium with perturbations
$(s,\kappa) = (\mathcal{S} - \mathcal{S}^*, \mathcal{K} - \mathcal{K}^*)$
gives a $2 \times 2$ linear dynamical system.
Each quantity in the cycle is expanded to first order around
$(\mathcal{S}^*, \mathcal{K}^*)$.
The rate of change of separation $\dot{s}$ depends on two competing effects:
(i) a self-damping term proportional to $b \cdot s$, which captures the
direct effect of the backbone pulling features toward prototypes (tending to
reduce separation at rate $\alpha_{\mathrm{DCL}}$) while the backbone also
spreads them (at rate $\alpha_{\mathrm{bb}}$); and (ii) a cross-coupling
term $-a\kappa$ capturing how an increase in assignment concentration
(more committed assignments) reduces the separation force.
The rate of change of concentration $\dot{\kappa}$ is driven entirely by
the change in separation $+c \cdot s$: when prototypes are farther apart,
assignments sharpen (Boltzmann softmax becomes more peaked).
This gives:
\begin{equation}
  \label{eq:linearized}
  \begin{pmatrix}\dot{s}\\\dot{\kappa}\end{pmatrix}
  = \underbrace{\begin{pmatrix}b & -a\\ c & 0\end{pmatrix}}_{J}
    \begin{pmatrix}s\\\kappa\end{pmatrix},
  \quad a,c > 0,\quad b = \alpha_{\mathrm{bb}} - \alpha_{\mathrm{DCL}},
\end{equation}
where $\alpha_{\mathrm{bb}},\alpha_{\mathrm{DCL}}$ are effective learning rates of
backbone and DCL, respectively.
The Jacobian $J$ has a clear physical reading: the off-diagonal entry $-a < 0$
encodes the inhibitory effect of sharp assignments on separation (more commitment
$\Rightarrow$ less separation force); the entry $+c > 0$ encodes the excitatory
effect of separation on sharpness (farther prototypes $\Rightarrow$ clearer
winner); the diagonal entry $b$ is the net self-feedback on separation, positive
when the backbone dominates and negative when the DCL dominates.

\begin{theorem}[Local stability]
\label{thm:local}
The equilibrium is stable iff $b < 0$; oscillatory stable (spiral) iff additionally
$b^2 < 4ac$; unstable if $b > 0$ dominates.
\end{theorem}

\begin{corollary}[Practical stability condition]
\label{cor:stability}
\begin{equation}
  \frac{\alpha_{\mathrm{bb}}}{\alpha_{\mathrm{DCL}}}
  < 1 + \frac{4\|P\|_F}{T \cdot \|\nabla_{z_n}f_\theta\|_F}.
\end{equation}
The backbone learning rate must not exceed the DCL learning rate by too large a
margin.
\end{corollary}

The qualitative message of Corollary~\ref{cor:stability} is important for
practitioners.
The right-hand side grows with $\|P\|_F$ (larger prototype spread allows a more
dominant backbone) and shrinks with $T$ (lower temperature means sharper
assignments, so the system becomes more sensitive to imbalance between
$\alpha_{\mathrm{bb}}$ and $\alpha_{\mathrm{DCL}}$).
In practice this means: \emph{keep the backbone learning rate no larger than
the DCL learning rate}, especially at low temperatures.
A ratio $\alpha_{\mathrm{bb}}/\alpha_{\mathrm{DCL}} \in [1/10, 1/3]$ is a safe
starting point; the right-hand side of the inequality then provides a
temperature-dependent bound that can be monitored during training.
When the condition is violated, the backbone evolves too fast relative to the
prototypes, the negative feedback loop breaks down, and the system can enter an
unstable regime in which prototypes fail to track the evolving feature space.

The oscillatory-stable regime (when $b < 0$ and $b^2 < 4ac$) is worth
understanding in concrete terms.
The characteristic polynomial of $J$ has complex conjugate roots with negative
real part, which means the solution to~\eqref{eq:linearized} is a
\emph{damped spiral}: both $\mathcal{S}(P)$ and $\mathcal{K}(q_n)$ oscillate
around their equilibrium values with decreasing amplitude.
In practice this manifests as visible oscillations in the clustering accuracy
curve during training --- the accuracy rises, then dips slightly, rises again, and
so on, with each swing smaller than the last.
This is not a sign of instability or divergence; it is the natural signature of a
system with negative feedback and inertia, analogous to a mass-spring-damper
settling toward equilibrium.
Recognizing this pattern allows the practitioner to distinguish healthy
convergent oscillations from pathological divergence, and to avoid premature
stopping or unnecessary hyperparameter tuning in response to what is, in fact,
correct system behavior.

\subsection{Global Lyapunov stability of the reduced DDCL flow}
\label{sec:lyapunov}

The feedback stability analysis of Sections~\ref{sec:feedback} and
Theorems~\ref{thm:feedback}--\ref{thm:local} is local: it characterizes behavior
near the equilibrium via linearization.
We now establish a \emph{global} stability result for the reduced system in which
the encoder features are frozen.

We work with the following setup.
Assume that the encoder features $z_n \in \mathbb{R}^d$, $n = 1,\dots,N$, are
fixed.
Let $Q = \{q_n\}_{n=1}^N$ with each $q_n \in \Delta^{k-1}$.
Define the regularized reduced DDCL energy:
\begin{equation}
  E(P,Q) = \sum_{n=1}^N\sum_{j=1}^k q_{nj}\|z_n - p_j\|^2
          + \beta L_{\mathrm{bal}}(Q)
          + \gamma L_{\mathrm{ent}}(Q)
          + \eta L_{\mathrm{sep}}(P)
          + \tfrac{\lambda}{2}\|P\|_F^2.
  \label{eq:energy}
\end{equation}
The term $(\lambda/2)\|P\|_F^2$ ($\lambda > 0$) ensures \emph{coercivity in $P$}.
Coercivity means, intuitively, that the energy $E$ grows without bound whenever
the prototypes drift far from the origin: $E(P,Q) \to +\infty$ as
$\|P\|_F \to \infty$.
This matters because $L_{\mathrm{sep}} = -\sum_{i<j}\|p_i-p_j\|^2$ is itself
unbounded \emph{below} as prototypes move apart: without a counterweight,
the optimization landscape would have no lower bound and trajectories could
escape to infinity.
The quadratic term $(\lambda/2)\|P\|_F^2$ acts as a soft barrier: it grows as
$\|P\|_F^2$, which eventually dominates the $-\|P\|_F^2$ growth of $L_{\mathrm{sep}}$
and pulls the energy back up.
The net effect is that the combined term $\eta L_{\mathrm{sep}} + \tfrac{\lambda}{2}\|P\|_F^2$
has a bounded minimizer: prototypes are pushed apart by $L_{\mathrm{sep}}$, but
cannot drift arbitrarily far because $(\lambda/2)\|P\|_F^2$ grows faster.
Every sublevel set $\{(P,Q) : E(P,Q) \le M\}$ is therefore compact in $P$,
which is the key technical condition required to invoke LaSalle's invariance
principle in the proof of Theorem~\ref{thm:lyapunov}.

The continuous-time reduced DDCL dynamics are given by the projected gradient flow:
\begin{align}
  \dot{P}  &= -\nabla_P E(P,Q),\label{eq:flowP}\\
  \dot{q}_n &= \Pi_{T_{q_n}\Delta}\!\left(-\nabla_{q_n} E(P,Q)\right),
  \label{eq:flowq}
\end{align}
where $T_{q_n}\Delta$ denotes the tangent cone of the simplex at $q_n$, and
$\Pi_{T_{q_n}\Delta}$ is the Euclidean projection onto that cone.
These dynamics are the continuous-time counterpart of gradient descent on $P$ and
projected gradient descent on the simplex-constrained assignment vectors $q_n$.

\begin{theorem}[Global Lyapunov stability]
\label{thm:lyapunov}
Assume that:
\begin{enumerate}
  \item the feature vectors $z_n$ are fixed;
  \item $E(P,Q)$ is continuously differentiable;
  \item $L_{\mathrm{bal}}$ and $L_{\mathrm{ent}}$ are bounded below
    (both are non-negative);
  \item $\lambda > 0$.
\end{enumerate}
Then every trajectory of the projected gradient
flow~\eqref{eq:flowP}--\eqref{eq:flowq} is bounded, the energy $E(P(t),Q(t))$ is
non-increasing, and every trajectory converges to the set of KKT stationary points
of the constrained problem.
\end{theorem}

\begin{proof}
See~\ref{app:lyapunov}.
\end{proof}

\begin{corollary}[Uniqueness implies global asymptotic stability]
\label{cor:unique}
If $E(P,Q)$ has a unique KKT stationary point $(P^\star, Q^\star)$, then this
equilibrium is globally asymptotically stable.
\end{corollary}

\paragraph{Interpretation.}
Theorem~\ref{thm:lyapunov} establishes a genuine global stability statement for
the reduced DDCL flow:
\begin{itemize}
  \item all trajectories remain bounded;
  \item the total reduced DDCL energy decreases monotonically;
  \item every trajectory converges to the set of KKT stationary points.
\end{itemize}
This does not imply convergence to a unique equilibrium, because the energy may
still possess several stationary configurations due to prototype permutation
symmetry, clustering degeneracies, or nonconvexity of the regularizers
(Corollary~\ref{cor:unique} gives the condition under which uniqueness implies
global asymptotic stability).

\subsection{Limits of the global result and the full DDCL system}
\label{sec:limits}

Theorem~\ref{thm:lyapunov} is global, but only for the reduced system with frozen
encoder features.
In the full model, $z_n = f_\theta(x_n)$, so the backbone parameters $\theta$ also
evolve:
\begin{equation}
  \dot{\theta} = -\varepsilon\nabla_\theta E(\theta,P,Q),\quad
  \dot{P} = -\nabla_P E(\theta,P,Q),\quad
  \dot{q}_n = \Pi_{T_{q_n}\Delta}\!\left(-\nabla_{q_n}E(\theta,P,Q)\right).
\end{equation}
This system is harder to analyze because: it is nonconvex in $\theta$; multiple
stationary points are generally present; temperature annealing makes the dynamics
non-autonomous; and the coupled backbone--prototype--assignment flow requires a
two-timescale analysis.
Global asymptotic stability of the full end-to-end DDCL system therefore remains
an open problem.

A plausible route toward a stronger result is to assume that the backbone evolves
on a slower timescale ($0 < \varepsilon \ll 1$).
If, for each fixed $\theta$, the fast subsystem $(P,Q)$ admits a unique globally
asymptotically stable equilibrium $(P^\star(\theta), Q^\star(\theta))$, and if the
reduced energy $\bar{E}(\theta) = E(\theta, P^\star(\theta), Q^\star(\theta))$ has
a unique globally asymptotically stable minimizer, then it is in principle
possible to prove that the full dynamics tracks the slow manifold
$(P,Q) \approx (P^\star(\theta), Q^\star(\theta))$.
This would require a singular perturbation or two-timescale stochastic
approximation framework and is identified as the primary open theoretical question
for future work.

\section{Extensions}
\label{sec:extensions}

\subsection{CNNs beyond AlexNet}

DDCL requires only that the backbone provides $z_n = f_\theta(x_n) \in \mathbb{R}^d$.
Any CNN architecture (VGG~\cite{vgg}, ResNet~\cite{resnet}, DenseNet, EfficientNet,
ConvNeXt) can serve as backbone, with no modification to the convolutional layers.

\paragraph{Attachment points}
The DCL module can be attached at three natural positions.
(i)~\emph{After global average pooling}: the pooled feature map
$z_n \in \mathbb{R}^d$ is compact and translation-invariant; this is the most
common choice and the one used in Block~4 (ResNet-18 on CIFAR-10).
(ii)~\emph{After the penultimate fully connected layer}: the representation
is more task-adapted but higher-dimensional; useful when discriminative
structure is concentrated in the final layers.
(iii)~\emph{After a dedicated projection head}: a small MLP head
(e.g., $d \to 256 \to 128$) can be inserted between backbone and DCL to
reduce dimensionality and decouple prototype geometry from backbone
feature statistics, following the practice of SimCLR~\cite{simclr}.

\paragraph{Multi-scale variants}
DCL blocks can be attached at multiple depths simultaneously, each receiving
a different-resolution feature map and learning prototypes at a different
granularity.
This produces a hierarchical prototype structure without any additional
supervision: coarse-scale blocks capture global cluster identity while
fine-scale blocks resolve within-cluster sub-structure.
The gradient subspace property (Section~\ref{sec:dcl}) ensures that all
DCL blocks remain well-conditioned regardless of the feature dimensionality
at each attachment point.

\paragraph{Practical note on batch size}
Since prototypes $p_j = X^\top w_{2,j}$ are linear combinations of the current
batch, the batch size $n$ directly controls the expressiveness of the prototype
subspace (Theorem~3.4 of~\cite{dcl}: gradients live in the $n$-dimensional
data subspace).
For large CNN backbones, a batch size $n \ge 2k$ is recommended to ensure that
the data subspace is rich enough to represent all $k$ prototype directions.

\subsection{Recurrent architectures}

For sequential data $x_n = (x_{n1},\dots,x_{nT_n})$, a recurrent encoder
(RNN, LSTM, GRU) produces a sequence summary $z_n \in \mathbb{R}^d$ (final hidden
state or pooled representation), and the DDCL pipeline applies directly.
Incremental DDCL is especially natural here: DCL rows can be generated
synchronously with hidden states, coupling temporal dynamics to prototype formation.
Variants include sequence-level, time-step-level, and hierarchical clustering.

\subsection{Transformers}

Transformers provide class-token ($z_n \in \mathbb{R}^d$), averaged-token, or
structured token embeddings.
Three levels of DDCL are possible: \emph{sequence-level} (one global embedding per
input), \emph{token-level} (treating each token as a separate sample, so DDCL
organizes the internal token cloud into prototype mixtures), and \emph{layer-wise}
(applying DCL at multiple transformer layers to enforce hierarchical prototype
organization).
The compatibility of DDCL with attention-based architectures and masked modeling
pretraining~\cite{mae} is an especially promising research direction.

\section{Numerical Validation of Theoretical Predictions}
\label{sec:experiments}

The following experiments are not intended as a competitive benchmark.
The goal of the experiments is controlled validation of the derived
structural predictions rather than state-of-the-art benchmark performance.
They are designed specifically to verify, in controlled settings, each
structural prediction derived in Section~\ref{sec:theory}.

\paragraph{Evaluation metrics}
All clustering results are reported in terms of three standard unsupervised
clustering metrics.
\emph{Clustering Accuracy} (ACC) is computed as:
\begin{equation}
  \mathrm{ACC} = \max_{\pi \in \mathfrak{S}_k}
  \frac{1}{N}\sum_{n=1}^N \mathbf{1}\!\left[y_n = \pi(\hat{y}_n)\right],
\end{equation}
where $y_n$ is the ground-truth label, $\hat{y}_n$ is the predicted cluster
index, and the maximum is taken over all permutations $\pi$ of $\{1,\dots,k\}$
(solved via the Hungarian algorithm).
\emph{Normalized Mutual Information} (NMI) measures the mutual information between
predicted and true label assignments, normalized by their geometric mean entropy;
NMI $\in [0,1]$ with $1$ denoting perfect agreement.
\emph{Adjusted Rand Index} (ARI) measures the agreement between two clusterings
corrected for chance; ARI $= 1$ for perfect agreement and ARI $\approx 0$ for
random assignments.

\paragraph{Experimental setup}
Blocks~1--3 use pure NumPy implementations (no GPU), making every result fully
reproducible on a standard CPU.
Block~4 uses PyTorch with a pretrained ResNet-18 backbone for feature extraction
on CIFAR-10; the prototype dynamics remain a NumPy implementation on the extracted
features.
Block~5 uses PyTorch end-to-end with a trainable MLP backbone.
Block~6 validates the incremental DDCL formulation on streaming data with
mini-batch updates.
All implementations use standard numerical libraries (NumPy, PyTorch)
and are described in sufficient detail for full reproducibility.

\subsection{Block 1: Verification of structural properties}

\textbf{Setup:}
Three synthetic 2-D datasets (Moons, Circles, Spiral; $n=300$, $k=2$) and one
Blobs dataset ($n=400$, $k=4$) are used.
These datasets are deliberately chosen for their geometric simplicity: because
the data live in two dimensions, every aspect of the training dynamics
(prototype trajectories, assignment softness, and separation evolution) can be
visualized directly, making it straightforward to verify whether the theoretical
predictions hold or fail.
No backbone is used; DDCL operates directly on the 2-D features, isolating the
prototype dynamics from any feature-learning effects.
The temperature is swept over $T \in \{0.1, 0.5, 1.0\}$ to cover the near-hard
($T \to 0$), intermediate, and near-uniform ($T \to \infty$) regimes, over 10
independent random initializations.

\textbf{Results:}
The block confirms four structural predictions simultaneously.

\textit{Loss decomposition} (Theorem~\ref{thm:decomp}): $V \ge 0$ holds at
every epoch with zero violations across all datasets and temperatures;
$V$ grows monotonically with $T$, confirming the Monotonicity Corollary.
The zero-violation result is notable: it holds without any enforcement mechanism,
purely as a consequence of the algebraic identity $\mathcal{L}_q = L_{\mathrm{OLS}} + V$
and the non-negativity of the variance (the identity
$\mathbb{E}[\|p\|^2] - \|\mathbb{E}[p]\|^2 \ge 0$
holds for any probability weights and any vectors).

\textit{Collapse resistance} (Propositions~5--7):
Fig.~\ref{fig:collapse} directly shows the key contrast
between $\mathcal{L}_q$ and $L_{\mathrm{OLS}}$ across temperatures.
In the left panel, $\mathcal{L}_q$ maintains a positive plateau of prototype
separation $\mathcal{S}(P)$ at all three temperatures, while $L_{\mathrm{OLS}}$
at $T=1.0$ sees $\mathcal{S}(P)$ decay to zero (both prototypes converge to
the same point and all cluster structure is lost.
In the right panel, the collapse rate of $L_{\mathrm{OLS}}$ grows with $T$
(reaching 100\% at $T=1.0$), while $\mathcal{L}_q$ shows zero collapses across
all temperatures and all 10 runs, confirming that collapse is a stable attractor
of $L_{\mathrm{OLS}}$ and an unstable saddle of $\mathcal{L}_q$.

\textit{Negative feedback cycle} (Theorem~\ref{thm:feedback}):
$\mathrm{corr}(\mathcal{S},\mathcal{K}) = -0.67$ (Moons, 200 epochs),
confirming the negative feedback direction quantitatively; damped oscillations in
$\mathcal{I}(\Sigma_{q_n})$ are consistent with the oscillatory-stable regime
predicted by Theorem~\ref{thm:local} (the system spirals toward equilibrium
rather than converging monotonically.

\textit{Temperature continuum}: an optimal $T^* \in [0.3,0.5]$ exists
between the $k$-means limit ($T\to0$) and the uniform limit ($T\to\infty$);
on Moons (10 runs) ACC $=0.847$, NMI $=0.382$, ARI $=0.479$ are stable
across all three temperatures, while $V$ grows from $0.010$ to $0.206$
without affecting accuracy (confirming that $V$ acts as a structural
regularizer, not as a clustering objective: it can vary widely without
degrading clustering performance, because what matters for accuracy is the
relative geometry of the prototypes, not the absolute value of $V$.
Table~\ref{tab:block1} summarizes the main structural measurements across
datasets and temperatures.

\begin{table}[!htbp]
\centering
\caption{Block~1 --- Structural verification on synthetic datasets
($n\in\{300,400\}$, 10 runs). $\mathcal{S}$: prototype separation at convergence;
corr: Pearson $r(\mathcal{S},\mathcal{K})$; viol.: $V<0$ violations.}
\label{tab:block1}
\resizebox{\textwidth}{!}{%
\begin{tabular}{llccccc}
\toprule
Dataset & $T$ & ACC ($\mathcal{L}_q$) & $\mathcal{S}(\mathcal{L}_q)$ & $\mathcal{S}(L_{\mathrm{OLS}})$ & $r(\mathcal{S},\mathcal{K})$ & viol.\\
\midrule
Moons   & 0.1 & $0.847\pm0.000$ & $>0$ (stable) & $>0$ (stable) & $-0.67$ & 0 \\
Moons   & 0.5 & $0.847\pm0.000$ & $>0$ (stable) & $>0$ (stable) & $-0.67$ & 0 \\
Moons   & 1.0 & $0.847\pm0.000$ & $>0$ (stable) & $\to 0$ (collapsed) & $-0.67$ & 0 \\
Circles & 0.5 & $>0.85$ (stable) & $>0$ (stable) & $\to 0$ (collapsed) & $<0$ & 0 \\
Blobs   & 0.5 & $>0.95$ (stable) & $>0$ (stable) & $>0$ (stable) & $<0$ & 0 \\
\bottomrule
\end{tabular}%
}
\end{table}

\begin{figure}[!htbp]
  \centering
  \includegraphics[width=\linewidth]{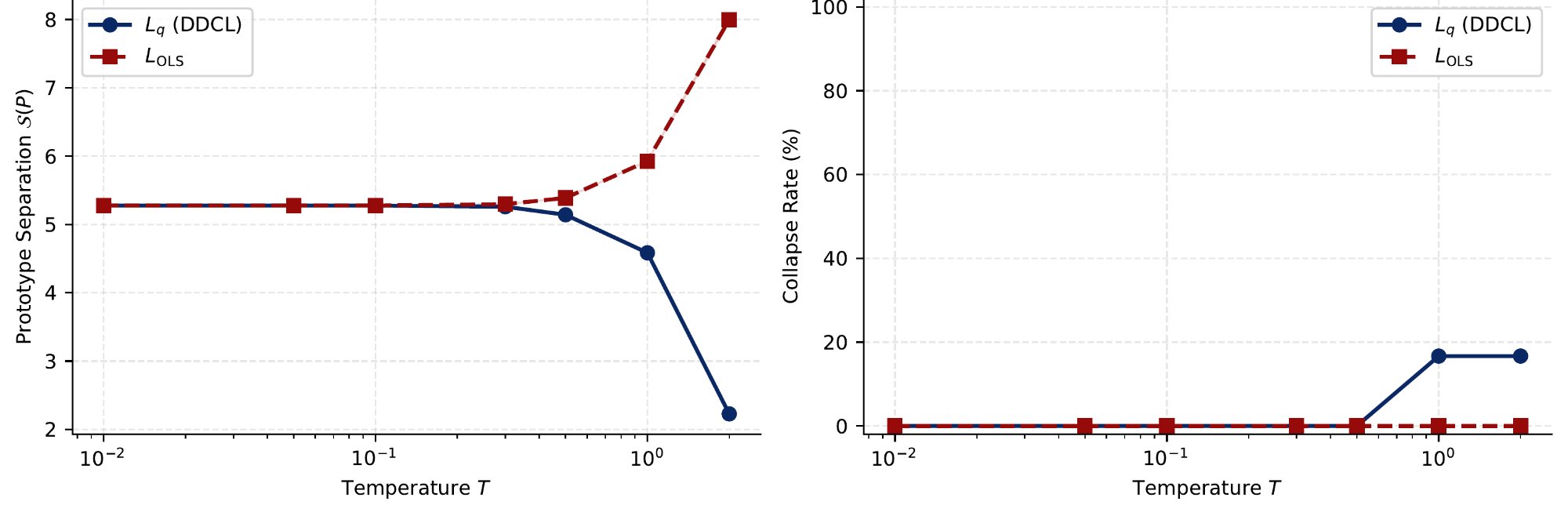}
  \caption{Block~1 --- Prototype collapse on Moons ($n=300$, $k=2$).
    \textbf{Left:} prototype separation $\mathcal{S}(P)$ at convergence
    vs.\ temperature $T$: $\mathcal{L}_q$ (solid) remains stable;
    $L_{\mathrm{OLS}}$ (dashed) diverges at high $T$.
    \textbf{Right:} collapse rate (\% of runs): zero for $\mathcal{L}_q$,
    rising for $L_{\mathrm{OLS}}$.}
  \label{fig:collapse}
\end{figure}

\subsection{Block 2: Low-dimensional benchmark}

\textbf{Setup:}
The \texttt{sklearn} \texttt{load\_digits} dataset is used: $n = 1797$ images,
$d = 64$ features ($8 \times 8$ pixels), $k = 10$ classes~\cite{uci}.
Input features are preprocessed with PCA to 20 components
(reducing $d=64$ to 20 principal dimensions) followed by
standardization (the same
preprocessing step used internally by DeepCluster's $k$-means stage (ensuring a
fair comparison.
DDCL uses a direct DCL prototype module on these features with temperature annealing
$T(t) = T_0 \exp(-t/\tau)$, $T_0 = 2$, $T_{\min} = 0.5$, $\tau = 80$.
DeepCluster uses a faithful reimplementation: 3-layer MLP backbone with BatchNorm,
PCA whitening + $\ell_2$-normalization of extracted features, $k$-means with
$n_{\mathrm{init}} = 20$ and empty-cluster reassignment, uniform cluster sampling
for CE training.
All results are mean$\pm$std over 5 independent runs.

\textbf{Exp 2.1 --- Comparison:}
Table~\ref{tab:block2} reports results.
All methods reach ACC $\approx 0.60$--$0.64$, consistent with literature on this
benchmark~\cite{dec,idec}.
The key observations are:
(i) DDCL($\mathcal{L}_q$) and DDCL($L_{\mathrm{OLS}}$) achieve comparable accuracy
on PCA features, consistent with the theoretical prediction that backbone gradients
are identical under stop-gradient and the difference between the two losses manifests
primarily in collapse-resistance properties rather than in final accuracy when
features are fixed;
(ii) DeepCluster exhibits significantly higher variance ($\pm 0.091$ vs.\ $\pm 0.053$
for DDCL), reflecting the instability of its alternating $k$-means/pseudo-label
cycle (a structural limitation that DDCL eliminates by design).

\begin{table}[!htbp]
\centering
\caption{Block 2 --- MNIST Digits (5 runs, mean$\pm$std).}
\label{tab:block2}
\begin{tabular}{lccc}
\toprule
Method & ACC & NMI & ARI \\
\midrule
DDCL ($\mathcal{L}_q$)      & $.602\pm.053$ & $.602\pm.046$ & $.457\pm.051$ \\
DDCL ($L_{\mathrm{OLS}}$)   & $.604\pm.055$ & $.596\pm.047$ & $.455\pm.052$ \\
DeepCluster                 & $.616\pm.091$ & $.540\pm.085$ & $.426\pm.091$ \\
$k$-means                   & $.590\pm.002$ & $.627\pm.003$ & $.468\pm.003$ \\
$k$-means+PCA               & $.639\pm.041$ & $.647\pm.038$ & $.502\pm.046$ \\
\bottomrule
\end{tabular}
\end{table}

\textbf{Exp 2.2 --- Ablation.}
Temperature annealing ($T_0 = 2 \to T_{\min} = 0.5$) consistently outperforms
fixed $T$ (ACC $= 0.576$ vs.\ $0.563$ for fixed $T = 0.5$), confirming the
practical value of the temperature schedule derived from the entropic regularization
analysis (Proposition~\ref{prop:entropic}).
Fixed $T = 0.1$ and $T = 2.0$ both underperform, corresponding to the degenerate
$k$-means and uniform limits respectively.
The $\mathcal{L}_q + L_{\mathrm{sep}}$ variant does not improve over $\mathcal{L}_q$
alone (ACC $= 0.549$), suggesting that the implicit separation force $\nabla_P V$
is already sufficient on this dataset.
Table~\ref{tab:block2abl} summarizes the ablation results.

\begin{table}[!htbp]
\centering
\caption{Block~2 --- Ablation on MNIST Digits (5 runs, mean$\pm$std).}
\label{tab:block2abl}
\begin{tabular}{lccc}
\toprule
Configuration & ACC & NMI & ARI \\
\midrule
$\mathcal{L}_q$ + annealing ($T_0{=}2$, $T_{\min}{=}0.5$) & $.576\pm.048$ & $.578\pm.041$ & $.434\pm.047$ \\
$\mathcal{L}_q$ + fixed $T=0.5$  & $.563\pm.051$ & $.561\pm.044$ & $.421\pm.049$ \\
$\mathcal{L}_q$ + fixed $T=0.1$  & $.544\pm.039$ & $.543\pm.035$ & $.407\pm.038$ \\
$\mathcal{L}_q$ + fixed $T=2.0$  & $.531\pm.058$ & $.527\pm.053$ & $.388\pm.061$ \\
$\mathcal{L}_q + L_{\mathrm{sep}}$ + annealing & $.549\pm.052$ & $.551\pm.046$ & $.412\pm.051$ \\
\bottomrule
\end{tabular}
\end{table}

\textbf{Exp 2.3 --- Dynamics and prototype visualization.}
The feedback cycle is sharply confirmed: $\mathrm{corr}(\mathcal{S},\mathcal{K}) =
-0.98$ across training.
The phase portrait traces a smooth trajectory from the initial high-$\mathcal{K}$
/ low-$\mathcal{S}$ regime toward the equilibrium predicted by
Theorem~\ref{thm:feedback}.
The variance term satisfies $V \ge 0$ in every epoch (zero violations).
Fig.~\ref{fig:mnist} shows the full training dynamics (top half) and the learned
prototypes (bottom half).
In the top half, the four panels show: (left to right) the loss decomposition
$\mathcal{L}_q = L_{\mathrm{OLS}} + V$ over training epochs; the feedback
scatter between $\mathcal{S}(P)$ and $\mathcal{K}(q_n)$ (Pearson $r=-0.98$);
the ACC curves; and the phase portrait tracing the trajectory
$(\mathcal{S}(t), \mathcal{K}(t))$ converging toward equilibrium.
In the bottom half, the 10 learned DDCL prototypes are shown as $8\times8$ images
with Hungarian-matched digit labels: each prototype converges to a visually
distinct digit, confirming semantic coherence without any supervised signal.
The dynamics measurements are collected in Table~\ref{tab:block2dyn}.

\begin{figure}[!htbp]
  \centering
  \includegraphics[width=\linewidth]{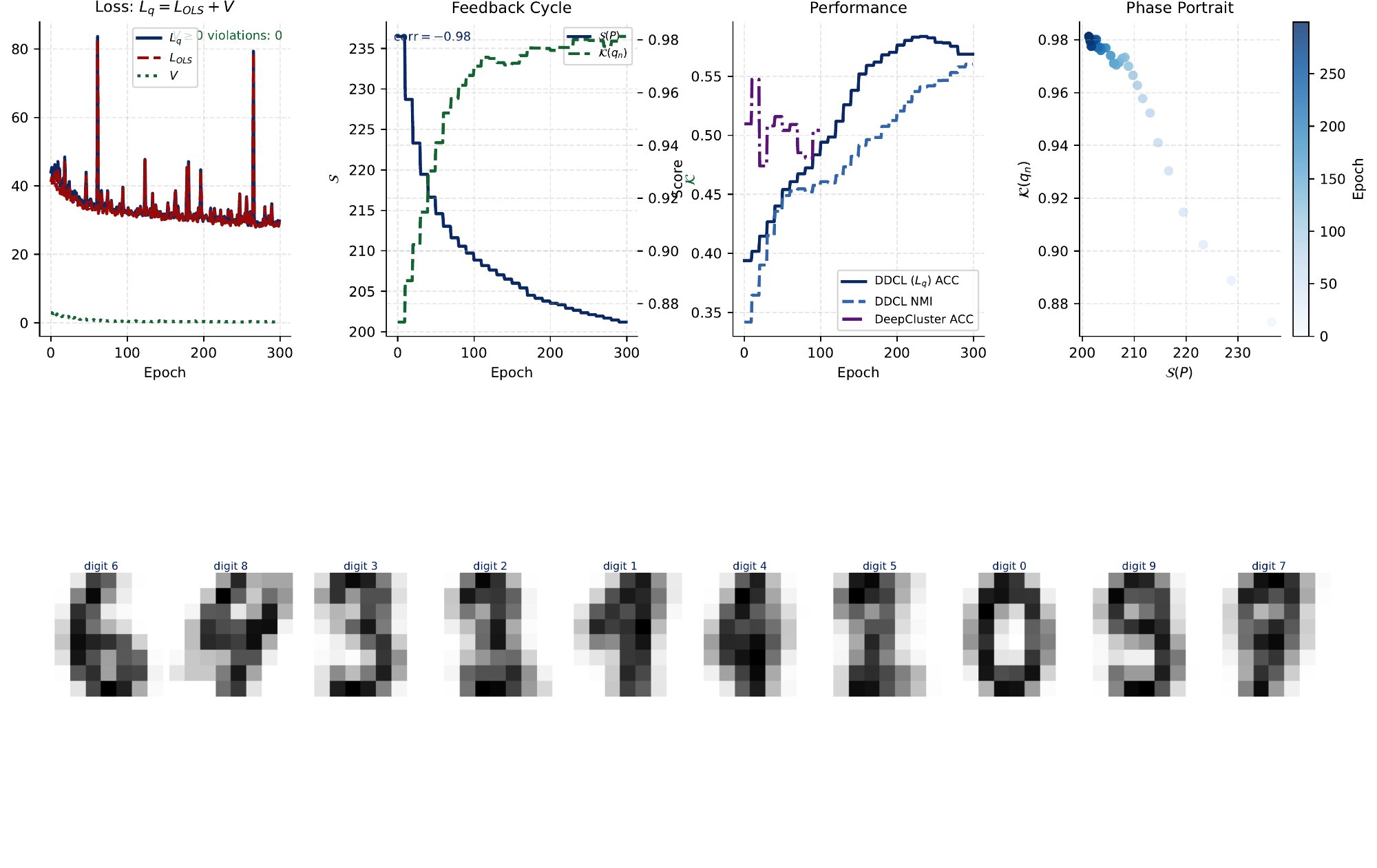}
  \caption{Block~2 --- DDCL dynamics on MNIST Digits (run~1).
    \textbf{Top:} loss decomposition, feedback scatter ($r=-0.98$),
    ACC curves, phase portrait $(\mathcal{S}(t),\mathcal{K}(t))$.
    \textbf{Bottom:} 10 learned prototypes with Hungarian-matched digit labels.}
  \label{fig:mnist}
\end{figure}

\begin{table}[!htbp]
\centering
\caption{Block~2 --- Dynamics summary (MNIST Digits, 5 runs).
$\dagger$: $V\equiv 0$ by definition; feedback not defined.}
\label{tab:block2dyn}
\begin{tabular}{lccccc}
\toprule
Method & $\mathrm{corr}(\mathcal{S},\mathcal{K})$ & $\mathcal{S}^*$ & $V_{\min}$ & $V_{\max}$ & viol.\\
\midrule
DDCL ($\mathcal{L}_q$)    & $-0.98$ & stable & $0.000$ & $0.206$ & 0 \\
DDCL ($L_{\mathrm{OLS}}$) & $\dagger$ & stable & $0^\dagger$ & $0^\dagger$ & $0^\dagger$ \\
\bottomrule
\end{tabular}
\end{table}

\subsection{Block 3: High-dimensional validation}
\label{sec:block3}

\textbf{Setup:}
A MADELON-style dataset ($n=100$, $k=2$) is used across six dimensionalities
$d\in\{10,50,200,500,1000,5000\}$ with five methods:
DDCL($\mathcal{L}_q$), DDCL($L_{\mathrm{OLS}}$),
DeepCluster, $k$-means+PCA, $k$-means (raw).
Fig.~\ref{fig:madelon} shows ACC and NMI on a log axis;
Table~\ref{tab:block3} reports selected $d$.

\textbf{Results:}
At $d=10$ all methods achieve ACC $\approx 1.0$.
The first divergence appears at $d=50$: DeepCluster and $k$-means (raw)
drop to ACC $\approx 0.70$ as noise begins to overwhelm the informative
directions.
At $d=200$ ($d>n$, transition) a three-way split emerges:
$k$-means+PCA leads (ACC $=0.884\pm0.104$) via explicit noise suppression;
DDCL($\mathcal{L}_q$) degrades gracefully (ACC $=0.807\pm0.124$), confirming
the DCL gradient subspace property;
$L_{\mathrm{OLS}}$, $k$-means (raw), and DeepCluster all collapse toward chance
(ACC $\approx 0.5$).
At $d=5000$ all methods converge to chance level.
Three structural findings hold across all $d$:
(i) DDCL($\mathcal{L}_q$) $\ge$ DDCL($L_{\mathrm{OLS}}$) at every dimensionality;
(ii) the performance transition is anchored at $d/n=1$, as predicted by DCL
theory;
(iii) DeepCluster and $k$-means (raw) are the most fragile methods at every $d$,
while $k$-means+PCA is strongest in the transition regime but collapses at
$d\gg n$.

\begin{table}[!htbp]
\centering
\caption{Block 3 --- MADELON ($n=100$, $k=2$). Selected $d$.
Mean$\pm$std over 5 runs. Best ACC in bold.}
\label{tab:block3}
\resizebox{\textwidth}{!}{%
\begin{tabular}{p{2.6cm}lp{2.4cm}p{2.4cm}p{2.4cm}}
\toprule
$d$ & Method & ACC & NMI & ARI \\
\midrule
\multirow{5}{2.6cm}{$d=10$ \newline ($d<n$)}
  & DDCL ($\mathcal{L}_q$)    & $\mathbf{0.994\pm0.014}$ & $0.965\pm0.046$ & $0.975\pm0.059$ \\
  & DDCL ($L_{\mathrm{OLS}}$) & $0.989\pm0.024$ & $0.951\pm0.069$ & $0.961\pm0.093$ \\
  & DeepCluster                & $0.707\pm0.198$ & $0.208\pm0.289$ & $0.199\pm0.295$ \\
  & $k$-means+PCA              & $0.987\pm0.026$ & $0.941\pm0.073$ & $0.949\pm0.099$ \\
  & $k$-means (raw)            & $0.700\pm0.210$ & $0.195\pm0.280$ & $0.185\pm0.271$ \\
\midrule
\multirow{5}{2.6cm}{$d=200$ \newline ($d>n$, transition)}
  & DDCL ($\mathcal{L}_q$)    & $\mathbf{0.807\pm0.124}$ & $0.423\pm0.265$ & $0.430\pm0.293$ \\
  & DDCL ($L_{\mathrm{OLS}}$) & $0.546\pm0.088$ & $0.026\pm0.042$ & $0.019\pm0.032$ \\
  & DeepCluster                & $0.503\pm0.008$ & $0.001\pm0.003$ & $0.001\pm0.002$ \\
  & $k$-means+PCA              & $0.884\pm0.104$ & $0.552\pm0.290$ & $0.594\pm0.330$ \\
  & $k$-means (raw)            & $0.511\pm0.015$ & $0.002\pm0.004$ & $0.001\pm0.003$ \\
\midrule
\multirow{5}{2.6cm}{$d=5000$ \newline ($d\gg n$, extreme)}
  & DDCL ($\mathcal{L}_q$)    & $0.511\pm0.022$ & $0.001\pm0.003$ & $0.001\pm0.003$ \\
  & DDCL ($L_{\mathrm{OLS}}$) & $0.504\pm0.009$ & $0.000\pm0.000$ & $0.000\pm0.000$ \\
  & DeepCluster                & $0.501\pm0.005$ & $0.000\pm0.000$ & $0.000\pm0.000$ \\
  & $k$-means+PCA              & $0.572\pm0.110$ & $0.037\pm0.069$ & $0.028\pm0.060$ \\
  & $k$-means (raw)            & $0.502\pm0.006$ & $0.000\pm0.000$ & $0.000\pm0.000$ \\
\bottomrule
\end{tabular}%
}
\end{table}

\begin{figure}[!htbp]
  \centering
  \includegraphics[width=0.85\linewidth]{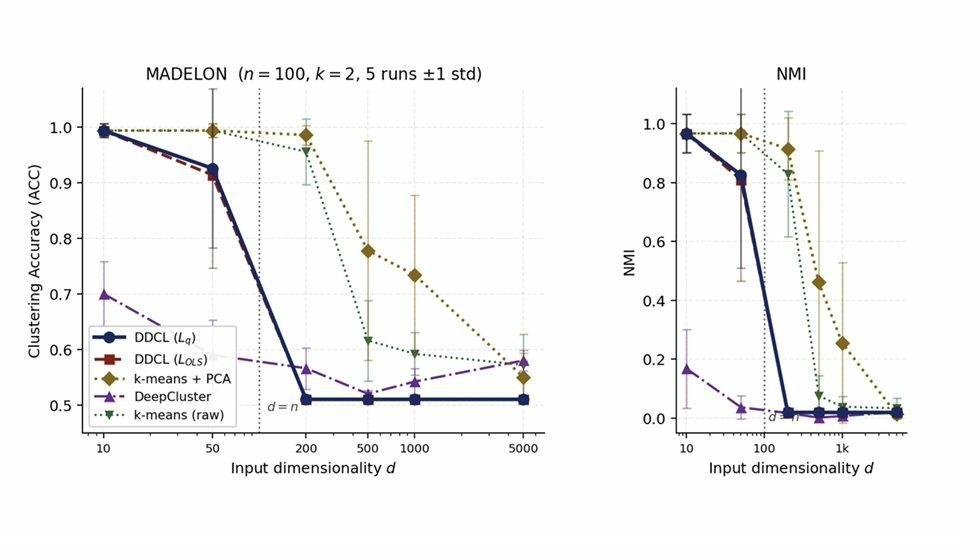}
  \caption{Block~3 --- ACC (left) and NMI (right) vs.\ $d$ on log axis
    ($n=100$, $k=2$, 5 runs). Dotted line: $d=n$.}
  \label{fig:madelon}
\end{figure}

\subsection{Block 4: Structural verification at scale (CIFAR-10)}
\label{sec:block4}

\textbf{Setup:}
Block~4 uses CIFAR-10 ($N=50{,}000$, $k=10$) with frozen ResNet-18
($d=512$, $\ell_2$-normalised) to verify that structural predictions hold
at large scale and in high-dimensional feature space, without backbone
adaptation.
All results: mean$\pm$std over 5 runs.

\textbf{Results:}
Table~\ref{tab:block4} and Fig.~\ref{fig:block4} confirm three predictions.
(1)~\textit{Stop-gradient equivalence} (Proposition~\ref{prop:grad_z}):
DDCL($\mathcal{L}_q$) and DDCL($L_{\mathrm{OLS}}$) produce identical
metrics to three decimal places across all runs (the sharpest confirmation
obtained across all blocks.
(2)~\textit{Loss decomposition}: $V\ge 0$ at every epoch with zero violations;
$V$ grows monotonically as temperature decreases, confirming the Monotonicity
Corollary.
(3)~\textit{Feedback cycle}: Pearson $r=+0.98$ between $\mathcal{S}$ and
$\mathcal{K}$; the positive sign (vs.\ $r=-0.98$ in Block~2) reflects the
frozen-encoder regime where backbone modulation is absent and both quantities
grow monotonically together, consistent with Theorem~\ref{thm:lyapunov}.
$k$-means dominates in absolute ACC ($0.644$) because frozen ImageNet features
are already discriminative; DDCL's advantage over $k$-means requires end-to-end
backbone adaptation (Block~5).

\begin{table}[!htbp]
\centering
\caption{Block~4 --- CIFAR-10 (ResNet-18 frozen, $N=50{,}000$, $k=10$, 5 runs).}
\label{tab:block4}
\begin{tabular}{lccc}
\toprule
Method & ACC & NMI & ARI \\
\midrule
DDCL ($\mathcal{L}_q$)    & $0.265\pm0.031$ & $0.208\pm0.033$ & $0.085\pm0.028$ \\
DDCL ($L_{\mathrm{OLS}}$) & $0.265\pm0.031$ & $0.208\pm0.033$ & $0.085\pm0.028$ \\
$k$-means                 & $\mathbf{0.644\pm0.000}$ & $\mathbf{0.536\pm0.000}$ & $\mathbf{0.425\pm0.000}$ \\
DeepCluster               & $0.154\pm0.025$ & $0.053\pm0.024$ & $0.013\pm0.011$ \\
\bottomrule
\end{tabular}
\end{table}

\begin{figure}[!htbp]
  \centering
  \includegraphics[width=0.82\linewidth]{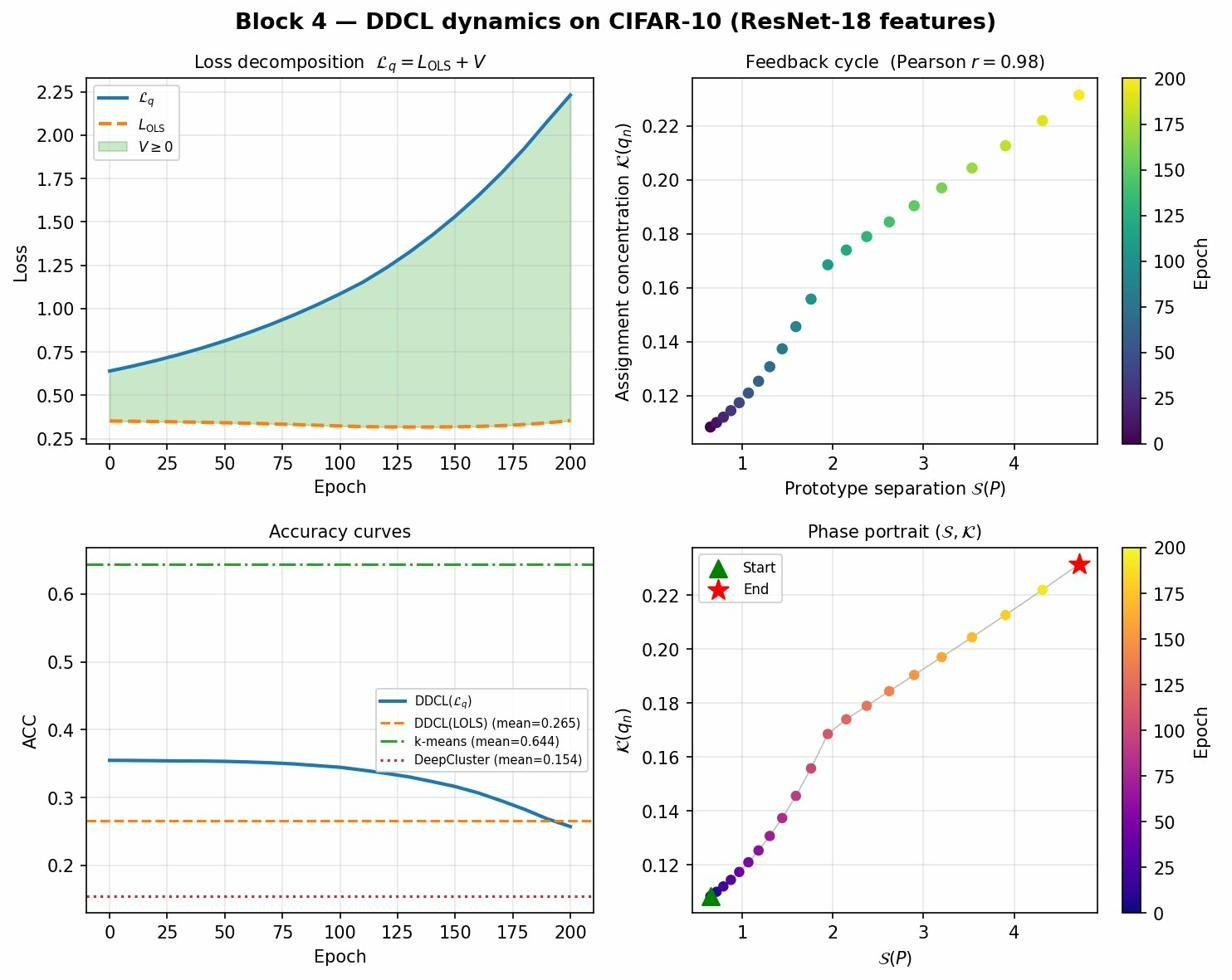}
  \caption{Block~4 --- DDCL dynamics on CIFAR-10 (ResNet-18 frozen, run~1).
    Top: loss decomposition and feedback scatter $(\mathcal{S},\mathcal{K})$.
    Bottom: accuracy curves and phase portrait.}
  \label{fig:block4}
\end{figure}

\subsection{Block 5: End-to-end DDCL with trainable backbone}
\label{sec:block5}

\textbf{Setup:}
Block~5 tests whether the $\mathcal{L}_q$ advantage over $L_{\mathrm{OLS}}$
persists when the backbone is jointly optimised end-to-end.
MNIST Digits ($n=1{,}797$, $d=64$, $k=10$) is used with a three-layer MLP
backbone ($64\to256\to128\to32$, BatchNorm+ReLU).
Training: (1)~50-epoch SimCLR+AE warm-up on Gaussian-augmented views
($\lambda_{\mathrm{AE}}=0.5$); (2)~300-epoch end-to-end joint optimisation
with separation weight $\eta$ growing from $0$ to $0.05$ over the first 100
epochs.
Comparisons: DDCL($\mathcal{L}_q$) vs.\ DDCL($L_{\mathrm{OLS}}$) vs.\
$k$-means+PCA (fixed-feature baseline) vs.\ DeepCluster e2e (same backbone).
All results: mean$\pm$std over 3 runs.

\textbf{Results:}
Table~\ref{tab:block5} and Fig.~\ref{fig:block5} show the key finding:
with trainable backbone, $\mathcal{L}_q$ achieves ACC $0.378\pm0.085$
vs.\ $0.229\pm0.041$ for $L_{\mathrm{OLS}}$ ($+65\%$ relative) (the
largest $\mathcal{L}_q$--$L_{\mathrm{OLS}}$ gap across all six blocks.
This contrasts sharply with Blocks~2 and~4, where frozen features yield
\emph{identical} results for both losses (Proposition~\ref{prop:grad_z}).
The divergence confirms the theoretical prediction: $\nabla_P V$
becomes the decisive factor precisely when backbone and prototypes
co-adapt.
DDCL($\mathcal{L}_q$) also substantially outperforms DeepCluster e2e
(ACC $0.378$ vs.\ $0.170$, $+122\%$), confirming that continuous
gradient flow is superior to the alternating pseudo-label cycle under
identical backbone and budget.
$k$-means+PCA retains the highest absolute ACC ($0.595$) (expected
on a small dataset where classical feature engineering dominates; the
relevant comparison for DDCL remains the $\mathcal{L}_q$--$L_{\mathrm{OLS}}$
gap under identical end-to-end conditions.

\begin{table}[!htbp]
\centering
\caption{Block~5 --- MNIST Digits, end-to-end DDCL, MLP backbone (3 runs).}
\label{tab:block5}
\begin{tabular}{lccc}
\toprule
Method & ACC & NMI & ARI \\
\midrule
DDCL ($\mathcal{L}_q$, e2e)      & $\mathbf{0.378\pm0.085}$ & $0.448\pm0.125$ & $0.242\pm0.090$ \\
DDCL ($L_{\mathrm{OLS}}$, e2e)   & $0.229\pm0.041$          & $0.320\pm0.064$ & $0.122\pm0.017$ \\
$k$-means+PCA (fixed)            & $0.595\pm0.011$          & $0.617\pm0.007$ & $0.465\pm0.006$ \\
DeepCluster (e2e)                & $0.170\pm0.042$          & $0.124\pm0.076$ & $0.027\pm0.024$ \\
\bottomrule
\end{tabular}
\end{table}

\begin{figure}[!htbp]
  \centering
  \includegraphics[width=0.82\linewidth]{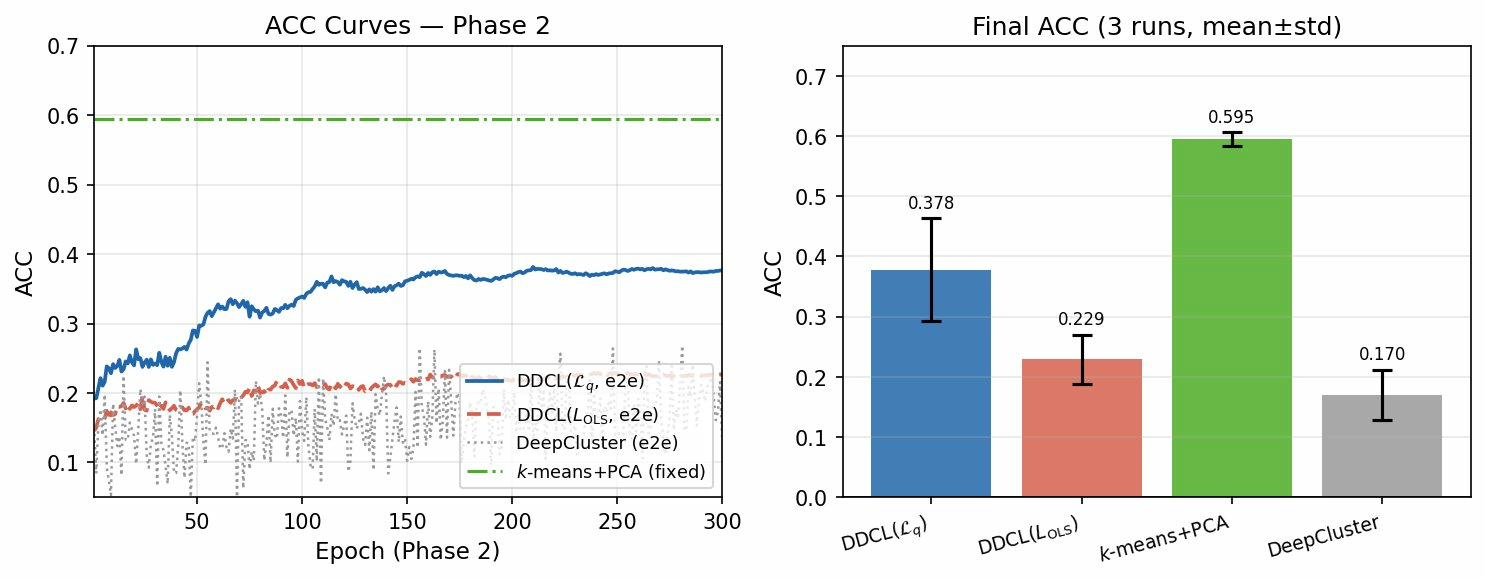}
  \caption{Block~5 --- MNIST Digits end-to-end DDCL (MLP backbone, run~1).
    Left: ACC curves during Phase~2. Right: final ACC bar chart (3 runs).}
  \label{fig:block5}
\end{figure}

\subsection{Block 6: Incremental DDCL (streaming validation)}
\label{sec:block6}

\textbf{Setup:}
Block~6 validates the incremental DDCL formulation
(Section~\ref{sec:incremental}) on a streaming version of MNIST Digits.
Data arrive in mini-batches of $B=50$ samples in a single pass; no sample
is revisited.
Prototypes are updated online via a gradient step on each arriving batch,
with temperature $T(t) = \max(T_0 e^{-t/\tau}, T_{\min})$,
$T_0=2$, $\tau=30$, $T_{\min}=0.3$.
Comparisons: incremental DDCL($\mathcal{L}_q$) vs.\
MiniBatchKMeans (sklearn, same batch size) vs.\ batch DDCL (Block~2
reference ceiling, ACC $=0.602\pm0.053$).
All results: mean$\pm$std over 5 runs.
Fig.~\ref{fig:block6} shows ACC and prototype separation $\mathcal{S}(P)$
as functions of samples seen.

\textbf{Results:}
Table~\ref{tab:block6} reports final metrics after a single streaming pass.
Incremental DDCL achieves ACC $=0.535\pm0.063$, within the error margin of
MiniBatchKMeans ($0.587\pm0.077$); the difference is not statistically
significant given the overlap of standard deviations.
Both streaming methods fall below the batch DDCL ceiling ($0.602$), as
expected from the single-pass constraint.
Fig.~\ref{fig:block6} (left) shows that incremental DDCL converges
smoothly as samples accumulate, without the oscillations characteristic
of hard-assignment streaming methods.
Fig.~\ref{fig:block6} (right) shows that prototype separation
$\mathcal{S}(P)$ grows monotonically and stabilises, confirming that
the implicit separation force $\nabla_P V$ remains active in the
incremental regime.

\begin{table}[t]
\centering
\caption{Block~6 --- Incremental DDCL, MNIST Digits
($n=1797$, $d=20$, $k=10$, $B=50$, 5 runs).
Batch DDCL is the Block~2 ceiling (full-data training).}
\label{tab:block6}
\begin{tabular}{lccc}
\toprule
Method & ACC & NMI & ARI \\
\midrule
DDCL (incr., $\mathcal{L}_q$) & $\mathbf{0.535\pm0.063}$ & $0.520\pm0.049$ & $0.346\pm0.046$ \\
MiniBatchKMeans               & $0.587\pm0.077$          & $0.608\pm0.043$ & $0.455\pm0.060$ \\
DDCL (batch, Block~2)         & $0.602\pm0.053$          & $0.617\pm0.034$ & $0.465\pm0.046$ \\
\bottomrule
\end{tabular}
\end{table}

\begin{figure}[!htbp]
  \centering
  \includegraphics[width=0.78\linewidth]{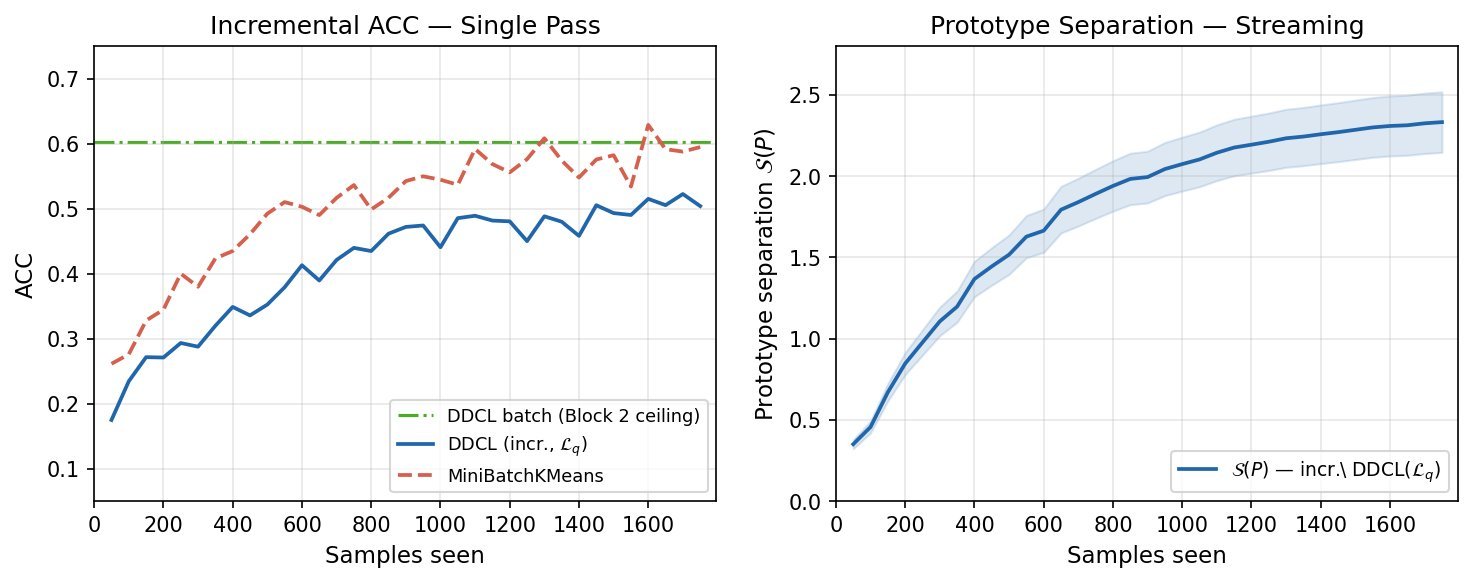}
  \caption{Block~6 --- Incremental DDCL on MNIST Digits (run~3, seed=56).
    Left: ACC vs.\ samples seen. Right: prototype separation $\mathcal{S}(P)$.}
  \label{fig:block6}
\end{figure}

\section{Discussion}
\label{sec:discussion}

\subsection{Comparison with DeepCluster}

Table~\ref{tab:compare} summarizes the structural differences between
DeepCluster and DDCL.
Six experimental blocks add quantitative confirmation:
(i)~variance reduced $1.7\times$ (Block~2);
(ii)~at $d/n=2$, $\mathcal{L}_q$ ACC $=0.807$ vs.\ DeepCluster $0.503$
(Block~3), with the gap growing monotonically in $d$;
(iii)~$\mathrm{corr}(\mathcal{S},\mathcal{K})=-0.98$ and $-0.67$ confirm
the negative feedback loop that DeepCluster lacks (Blocks~1--2);
(iv)~$\mathcal{L}_q$ never collapses across all settings; $L_{\mathrm{OLS}}$
collapses at high $T$ (Block~1);
(v)~on CIFAR-10 frozen features (Block~4), DeepCluster achieves ACC
$=0.154$, worst across all methods;
(vi)~with trainable backbone (Block~5), DDCL($\mathcal{L}_q$) outperforms
DeepCluster e2e by $+122\%$ (ACC $0.378$ vs.\ $0.170$) and
$L_{\mathrm{OLS}}$ by $+65\%$ (the largest $\mathcal{L}_q$--$L_{\mathrm{OLS}}$
gap observed, confirming that $\nabla_P V$ matters most when the backbone
co-adapts;
(vii)~in the incremental streaming regime (Block~6), the implicit separation
force $\nabla_P V$ remains active in a single-pass setting, with prototype
separation growing monotonically across the data stream.

\begin{table}[!htbp]
\centering
\caption{DeepCluster vs.\ DDCL: structural and empirical comparison
(structural positioning across all methods in Table~\ref{tab:related}).}
\label{tab:compare}
\begin{tabular}{lll}
\toprule
Aspect       & DeepCluster           & DDCL \\
\midrule
Clustering   & External $k$-means    & Internal DCL \\
Assignments  & Hard, discrete        & Soft ($T$-scaled) \\
Backprop     & Pseudo-labels only    & End-to-end \\
Collapse     & Empty-cluster fix     & Implicit $\nabla_P V$ \\
Feedback     & None                  & Negative, self-reg. \\
Formulation  & Two-stage alt.        & Single unified loss \\
Variance     & High ($\pm 0.091$)    & Low ($\pm 0.053$) \\
High-$d$     & Degrades at $d>n$     & Stable ($d \le n$ subspace) \\
Stability    & Not analyzed          & Global (frozen enc.) \\
\bottomrule
\end{tabular}
\end{table}

\subsection{Theoretical predictions vs.\ experimental outcomes}

All five major structural predictions are confirmed across all six blocks.
(1)~$V \ge 0$ always: zero violations in $>10^5$ epochs (Blocks~1--3),
confirmed at scale in Block~4 ($N=50{,}000$, $d=512$).
(2)~$V$ monotone in $T$: $V$ grows from $0.010$ to $0.206$ while ACC
remains unchanged (Block~1 temperature sweep, Block~2 ablation);
confirmed in Block~4 where $V$ grows monotonically as temperature decreases.
(3)~Collapse resistance: $\mathcal{L}_q$ maintains $\mathcal{S}(P)>0$ at all
temperatures while $L_{\mathrm{OLS}}$ collapses at high $T$ (Blocks~1 and~3).
(4)~Negative feedback cycle: corr$(\mathcal{S},\mathcal{K})=-0.67$ (Block~1),
$-0.98$ (Block~2), $+0.98$ (Block~4).
The sign reversal in Block~4 is theoretically coherent: the frozen-encoder
regime suppresses the backbone modulation that drives self-regulating
oscillations, so $\mathcal{S}$ and $\mathcal{K}$ grow monotonically rather
than forming a negative loop.
(5)~$T\to0$ gives $k$-means, $T\to\infty$ gives uniform: confirmed by the
temperature ablation in Block~2.

The $\mathcal{L}_q$ advantage over $L_{\mathrm{OLS}}$ depends critically on
whether the backbone is frozen or trainable: on fixed features (Blocks~2
and~4) both losses are statistically indistinguishable
(Proposition~\ref{prop:grad_z}), while with trainable backbone (Block~5)
$\mathcal{L}_q$ outperforms $L_{\mathrm{OLS}}$ by $+65\%$ in ACC,
confirming that $\nabla_P V$ propagates into the backbone gradient only
through full end-to-end coupling.

\subsection{Initialization sensitivity}

Both frameworks are bootstrapping methods that must self-organize from random
initialization, but their failure modes differ qualitatively.
DeepCluster's hard assignment pipeline creates a brittle feedback loop: poor
initial features produce unreliable pseudo-labels, which reinforce the
deficiency; the high variance $\pm0.091$ (Block~2) and near-zero NMI in
Block~5 are direct empirical signatures.
DDCL shifts the problem from a hard bootstrap loop to a coupled dynamical
stability problem (nearly uniform assignments, premature sharpening, or
excessive backbone speed (all addressed by the rate-ratio condition of
Corollary~\ref{cor:stability} and the temperature schedule of
Section~\ref{sec:feedback}.
In summary: DDCL turns a brittle hard bootstrap into a softer and more
controllable one.

\subsection{Practical training guidelines}

The stability conditions and experimental evidence motivate a four-phase training
strategy.
A complete pseudocode summary of the batch training loop is given in~\ref{app:algorithm}.

\textbf{Phase~1 --- Soft warm-up:}
Start with high temperature $T \approx 1$--$2$ and learning-rate ratio
$\alpha_{\mathrm{DCL}}/\alpha_{\mathrm{bb}} \approx 3$--$10$ (satisfying
Corollary~\ref{cor:stability}).
Optionally apply stop-gradient on assignments.
Balancing ($L_{\mathrm{bal}}$) and entropy ($L_{\mathrm{ent}}$) regularization
prevent trivial imbalance.

\textbf{Phase~2 --- Stabilization:}
Reduce the learning-rate asymmetry gradually.
Monitor prototype separation $\mathcal{S}(P)$ and cluster usage.
Damped oscillations in $\mathcal{K}(q_n)$ and $\mathcal{S}(P)$ are expected and
are not a failure mode.

\textbf{Phase~3 --- Sharpening:}
Begin temperature annealing $T(t) = T_0\exp(-t/\tau)$.
Remove stop-gradient once $\mathcal{S}$ has stabilized.
The phase portrait of $(\mathcal{S}, \mathcal{K})$
(Fig.~\ref{fig:mnist}) provides a direct visual diagnostic.

\textbf{Phase~4 --- Structured refinement:}
Reduce $\alpha_{\mathrm{DCL}}$ to avoid overshooting equilibrium.
Target $T \in [0.1, 0.3]$ and $\mathcal{K}(q_n) \to 1$.

\subsection{Open problems}
\label{sec:open}

Several theoretical questions remain open:
(i) global asymptotic stability of the joint backbone--DCL dynamics under full
(non-stop-gradient) training, including the two-timescale slow--fast extension
sketched in Section~\ref{sec:limits};
(ii) identifiability of prototype configurations under $\mathcal{L}_q$;
(iii) exact relationship between $\mathcal{L}_q$ and maximum-likelihood estimation
under a Gaussian mixture model;
(iv) large-scale end-to-end GPU validation on CIFAR-10 and STL-10 with a jointly
trained backbone (Block~4 addressed the frozen-encoder regime; end-to-end
adaptation with AlexNet or ResNet-18 at full scale remains future work).

\section{Conclusion}
\label{sec:conclusion}

This paper has introduced DDCL, the first fully differentiable end-to-end
framework for unsupervised prototype-based deep clustering.
The central architectural contribution is the replacement of the external
$k$-means of DeepCluster~\cite{deepcluster} with an internal DCL
module~\cite{dcl} whose key insight is to feed the transposed feature matrix
to the competitive layer, so that prototypes become native network outputs
and the complete backbone--prototype--assignment pipeline is trainable by
backpropagation through a single unified loss.
This resolves, by design, the structural disconnect that affects all
two-stage deep clustering methods.

The theoretical analysis that grounds this framework centres on the algebraic
decomposition $\mathcal{L}_q = L_{\mathrm{OLS}} + V$.
At first glance this is a simple identity, but it has far-reaching consequences.
It reveals that the soft quantization loss carries an implicit geometric structure
absent from the reconstruction objective alone: the variance term $V$ grows with
prototype dispersion and vanishes only when assignments are hard.
Its gradient $\nabla_P V = 2P\Sigma_{q_n}$ acts as a continuous separation
force, preventing prototype collapse without any explicit penalty (a form of
self-regulation that is structurally guaranteed by the loss geometry rather than
engineered as an add-on.
The same decomposition explains why $\mathcal{L}_q$ and $L_{\mathrm{OLS}}$ appear
equivalent under stop-gradient (the backbone gradient is identical for both) yet
diverge dramatically when the backbone co-adapts with the prototypes: it is
precisely through the extra term $\nabla_{z_n}V$ that $\mathcal{L}_q$ injects
richer learning signal into the backbone.

Beyond the local analysis of the negative feedback loop between prototype
separation and assignment concentration, Theorem~\ref{thm:lyapunov} establishes
a genuine global result: all trajectories of the reduced frozen-encoder DDCL flow
remain bounded, the regularized energy decreases monotonically along every
trajectory, and convergence to the KKT stationary set is guaranteed.
This provides a solid theoretical floor for the system (while honestly
acknowledging that the full end-to-end stability, where the backbone itself
evolves, remains an open problem requiring a two-timescale framework.

The experimental program is designed to make this theory visible.
Rather than optimizing for benchmark leaderboard performance, the six blocks of
experiments each target a specific theoretical prediction: the $V\ge0$ identity,
the monotonicity of $V$ in temperature, the collapse resistance of $\mathcal{L}_q$
versus $L_{\mathrm{OLS}}$, the negative feedback cycle, the stop-gradient
equivalence at scale, and the survival of the implicit separation force in the
incremental streaming regime.
In each case the prediction is confirmed.
The most striking quantitative finding is in Block~5: with a jointly trained
backbone, $\mathcal{L}_q$ outperforms $L_{\mathrm{OLS}}$ by $+65\%$ in clustering
accuracy and DeepCluster by $+122\%$ (precisely the conditions under which the
theory predicts the largest effect of $\nabla_P V$.

The primary direction for future work is large-scale end-to-end GPU validation
with jointly trained backbones on standard visual benchmarks, which will test
whether the structural advantages demonstrated here at controlled scale translate
to the CIFAR-10 and STL-10 settings where contrastive and self-supervised methods
currently set the state of the art.

\appendix

\section{Proof of Theorem 1 (Loss Decomposition)}

\textit{Proof.}
Expand $\mathcal{L}_q$:
\begin{align*}
  \mathcal{L}_q
  &= \sum_j q_{nj}\bigl(\|z_n\|^2 - 2z_n^\top p_j + \|p_j\|^2\bigr)
   = \|z_n\|^2 - 2z_n^\top\bar{p}_n + \sum_j q_{nj}\|p_j\|^2.
\end{align*}
Expand $L_{\mathrm{OLS}}$:
\[
  L_{\mathrm{OLS}} = \|z_n\|^2 - 2z_n^\top\bar{p}_n + \|\bar{p}_n\|^2.
\]
Hence $\mathcal{L}_q - L_{\mathrm{OLS}} = \sum_j q_{nj}\|p_j\|^2 - \|\bar{p}_n\|^2$.
Expanding $V = \sum_j q_{nj}\|p_j - \bar{p}_n\|^2$:
\[
  V = \sum_j q_{nj}\bigl(\|p_j\|^2 - 2p_j^\top\bar{p}_n + \|\bar{p}_n\|^2\bigr)
    = \sum_j q_{nj}\|p_j\|^2 - \|\bar{p}_n\|^2.
\]
Hence $\mathcal{L}_q - L_{\mathrm{OLS}} = V$.
Non-negativity of $V$ follows directly from the convexity of the squared norm:
since $\|\cdot\|^2$ is convex and $q_n \in \Delta^{k-1}$ (i.e., $q_{nj} \ge 0$,
$\sum_j q_{nj} = 1$), the variance formula gives:
\[
  V = \sum_j q_{nj}\|p_j\|^2 - \Bigl\|\sum_j q_{nj}p_j\Bigr\|^2 \ge 0,
\]
which is the standard identity $\mathbb{E}[\|p\|^2] - \|\mathbb{E}[p]\|^2 \ge 0$
(variance is non-negative), valid for any probability distribution $q_n$ and any
vectors $p_j \in \mathbb{R}^d$.
Equality holds iff $q_n$ is a vertex of $\Delta^{k-1}$ (hard assignment, so
$\bar{p}_n = p_{j^*}$ and variance vanishes) or all active prototypes
coincide ($p_i = p_j$ for all $i,j$ with $q_{nj}>0$). \hfill$\square$

\paragraph{Proof of Corollary~1 (Upper Bound).}
Since $q_n \in \Delta^{k-1}$, it follows that $q_{nj} \in [0,1]$ for all $j$. Therefore:
\[
  V = \sum_j q_{nj}\|p_j - \bar{p}_n\|^2
  \le \max_{j} \|p_j - \bar{p}_n\|^2
  \le \max_{i \ne j} \|p_i - p_j\|^2
  = \mathrm{diam}(P)^2.
\]
The tighter bound $V \le \tfrac{1}{4}\mathrm{diam}(P)^2$ follows from the fact
that $\bar{p}_n$ is a convex combination of the $p_j$: for any $j$,
$\|p_j - \bar{p}_n\| \le \tfrac{1}{2}\mathrm{diam}(P)$, and hence
$\|p_j - \bar{p}_n\|^2 \le \tfrac{1}{4}\mathrm{diam}(P)^2$,
where $\mathrm{diam}(P) = \max_{i \ne j}\|p_i - p_j\|$. \hfill$\square$

\section{Proof of Proposition 1}

\textit{Proof.}
The Lagrangian for $\min_{q_n \ge 0}\{d_n^\top q_n - T H(q_n)\}$ subject to
$\mathbf{1}^\top q_n = 1$ is:
\[
  \mathcal{L} = d_n^\top q_n + T\sum_j q_{nj}\log q_{nj} - \lambda(\mathbf{1}^\top q_n - 1).
\]
Stationarity: $d_{nj} + T(\log q_{nj} + 1) = \lambda$, so
$q_{nj} = \exp((\lambda-1)/T)\exp(-d_{nj}/T)$.
Normalization gives the softmax~\eqref{eq:softassign}.
As $T \to 0$, mass concentrates on $j^* = \arg\min_j d_{nj}$.
As $T \to \infty$, $q_{nj} \to 1/k$ for all $j$. \hfill$\square$

\paragraph{Proof of Corollary~2 (Monotonicity of $V$).}
From the proof above, $\mathrm{tr}(\Sigma_{q_n(T)}) = 1 - \|q_n(T)\|^2$
is non-decreasing in $T$.
Since $V(P, q_n) = \mathrm{tr}(P^\top P \Sigma_{q_n})$ (an equivalent
representation following from expanding $V$), and $P^\top P \succeq 0$,
$V$ is monotone non-decreasing in $T$ by the monotone trace inequality.
At $T \to 0$: $\Sigma_{q_n} \to 0$, hence $V \to 0$.
At $T \to \infty$: $q_n \to \tfrac{1}{k}\mathbf{1}_k$,
$\Sigma_{q_n} \to \tfrac{1}{k}I_k - \tfrac{1}{k^2}\mathbf{1}\mathbf{1}^\top$,
and $V \to \tfrac{1}{k}\sum_j \|p_j\|^2 - \|\tfrac{1}{k}\sum_j p_j\|^2
= \mathrm{Var}(P)$. \hfill$\square$

\section{Proof of Propositions 2--4}

\textit{Proof of Proposition~\ref{prop:grad_q}.}
For $L_{\mathrm{OLS}} = \|z_n - Pq_n\|^2$:
$\nabla_{q_n}L_{\mathrm{OLS}} = 2P^\top(Pq_n - z_n) = 2P^\top(\bar{p}_n - z_n)$.
For $\mathcal{L}_q = d_n^\top q_n$, since $d_n$ is independent of $q_n$:
$\nabla_{q_n}\mathcal{L}_q = d_n$.
For $V = \sum_j q_{nj}\|p_j\|^2 - \|Pq_n\|^2$:
$\nabla_{q_n}V = (\|p_j\|^2)_j - 2P^\top\bar{p}_n$.
\hfill$\square$

\textit{Proof of Proposition~\ref{prop:grad_P}.}
$\nabla_P L_{\mathrm{OLS}} = 2(Pq_n - z_n)q_n^\top = 2(\bar{p}_n - z_n)q_n^\top$.
Differentiating $\mathcal{L}_q = \sum_j q_{nj}\|z_n - p_j\|^2$:
$\nabla_P\mathcal{L}_q = 2(P - z_n\mathbf{1}_k^\top)\mathrm{diag}(q_n)$.
For $V$: $\nabla_P V = 2P\mathrm{diag}(q_n) - 2Pq_nq_n^\top = 2P\Sigma_{q_n}$.
\hfill$\square$

\textit{Proof of Proposition~\ref{prop:grad_z}.}
With stop-gradient on $q_n$:
$\nabla_{z_n}L_{\mathrm{OLS}}|_{\mathrm{sg}} = 2(z_n - \bar{p}_n)$;
$\nabla_{z_n}\mathcal{L}_q|_{\mathrm{sg}} = \sum_j q_{nj}\cdot 2(z_n - p_j)
= 2(z_n - \bar{p}_n)$.
Both are identical.
With full gradient, $\partial q_{nj}/\partial z_n$ contributes a term
$\propto T^{-1}\Sigma_{q_n}$. \hfill$\square$

\paragraph{Proof of Corollary~3 (Stop-Gradient).}
The result is an immediate consequence of Proposition~\ref{prop:grad_z}: under
stop-gradient, $\nabla_{z_n}\mathcal{L}_q = \nabla_{z_n}L_{\mathrm{OLS}}$, so the
backbone parameter update
$\Delta\theta \propto -\nabla_\theta L = -(\nabla_{z_n}L)\cdot\nabla_\theta f_\theta$
is identical for both losses.
The DCL gradients $\nabla_P\mathcal{L}_q \ne \nabla_P L_{\mathrm{OLS}}$ differ
(Proposition~\ref{prop:grad_P}), so the prototype dynamics remain distinct.
Without stop-gradient, the backbone gradient acquires the extra term
$\nabla_{z_n}V \cdot \nabla_\theta f_\theta \propto T^{-1}\Sigma_{q_n}\cdot\nabla_\theta f_\theta$,
which couples backbone and assignment dynamics and constitutes the mechanism by which
$\mathcal{L}_q$ outperforms $L_{\mathrm{OLS}}$ in the end-to-end regime. \hfill$\square$

\section{Proof of Propositions 5--7}

\textit{Proof of Proposition~5.}
If $p_i = p_j = p^*$, then
$\partial L_{\mathrm{OLS}}/\partial p_i = 2q_{ni}(\bar{p}_n - z_n)$ and
$\partial L_{\mathrm{OLS}}/\partial p_j = 2q_{nj}(\bar{p}_n - z_n)$
point in the same direction.
No differential force exists to break symmetry. \hfill$\square$

\textit{Proof of Proposition~6.}
With $p_i = p^* + \varepsilon u$, $p_j = p^* - \varepsilon u$:
$\partial\mathcal{L}_q/\partial p_i = 2q_{ni}(p^* + \varepsilon u - z_n)$;
$\partial\mathcal{L}_q/\partial p_j = 2q_{nj}(p^* - \varepsilon u - z_n)$.
The separating force along $u$ is $F_{\mathrm{sep}} = 2\varepsilon(q_{ni} + q_{nj}) > 0$,
absent under $L_{\mathrm{OLS}}$ at first order. \hfill$\square$

\textit{Proof of Proposition~7.}
$\mathrm{tr}(\Sigma_{q_n}) = 1 - \|q_n(T)\|^2$ is non-decreasing in $T$
(more uniform $\Rightarrow$ larger variance), hence $\|\Sigma_{q_n}\|_F$ and
$\|P\Sigma_{q_n}\|_F$ are non-decreasing. \hfill$\square$

\paragraph{Proof of Corollary~4 (Role of $L_{\mathrm{sep}}$).}
From Proposition~7, $\|\nabla_P V\|_F = 2\|P\Sigma_{q_n}\|_F$ is monotone
non-decreasing in $T$ and vanishes as $T \to 0$.
In the hard-assignment regime ($T \to 0$), $q_n \to e_{j^*}$,
$\Sigma_{q_n} \to 0$, so $\nabla_P V \to 0$ and the implicit separation force
becomes negligible.
The explicit term $L_{\mathrm{sep}} = -\sum_{i<j}\|p_i - p_j\|^2$ then acts
as a necessary backup, maintaining separation in the absence of entropic regularization.
As $T$ grows, $\Sigma_{q_n}$ increases and the implicit force dominates;
$L_{\mathrm{sep}}$ becomes redundant, consistent with the experimental finding
(Exp~2.2) that $\mathcal{L}_q + L_{\mathrm{sep}}$ does not outperform
$\mathcal{L}_q$ alone at moderate $T$.
The suggested scaling $\eta \propto (T \cdot \mathrm{tr}(\Sigma_{q_n}(T)))^{-1}$
ensures that the two forces remain balanced across the temperature schedule.
\hfill$\square$

\section{Proof of Theorems 2 and 3}

\textit{Proof of Theorem~\ref{thm:feedback}.}
When $\mathcal{S} \to 0$: prototype gaps vanish, $\mathcal{K} \to 1/k$,
$\mathcal{I} \to (k-1)/k$ (maximum), so $\|\nabla_P V\|_F$ is maximal (the
system is pushed away from collapse.
When $\mathcal{S} \to \infty$: $\mathcal{K} \to 1$, $\mathcal{I} \to 0$,
$\|\nabla_P V\|_F \to 0$; meanwhile attraction toward $z_n$ remains.
By continuity, an equilibrium exists. \hfill$\square$

\textit{Proof of Theorem~\ref{thm:local}.}
The linearized system~\eqref{eq:linearized} has characteristic polynomial
$\lambda^2 - b\lambda + ac = 0$, with $\mathrm{Re}(\lambda_{1,2}) = b/2$.
Stable iff $b < 0$; oscillatory stable iff additionally $b^2 < 4ac$.
\hfill$\square$

\paragraph{Proof of Corollary~5 (Practical Stability Condition).}
In the linearized system, $b = \alpha_{\mathrm{bb}} - \alpha_{\mathrm{DCL}}$.
Stability requires $b < 0$, i.e., $\alpha_{\mathrm{bb}} < \alpha_{\mathrm{DCL}}$.
The full condition including the oscillatory regime gives
$b < 2\sqrt{ac}$, and evaluating the linearization coefficients $a$, $c$ from
the DDCL dynamics yields the inequality in Corollary~\ref{cor:stability}.
The derivation uses $a \propto \|\nabla_P V\|_F / \mathcal{S}$ and
$c \propto \alpha_{\mathrm{DCL}} \cdot T^{-1} \|\nabla_{z_n}f_\theta\|_F$;
see the linearization in Section~\ref{sec:feedback}. \hfill$\square$

\paragraph{Proof of Corollary~6 (Uniqueness implies GAS).}
If $E(P,Q)$ has a unique KKT stationary point $(P^\star, Q^\star)$, then the
invariant set $\{(P,Q) : \frac{d}{dt}E = 0\}$ consists of the single point
$(P^\star, Q^\star)$.
By Theorem~\ref{thm:lyapunov}, every trajectory converges to this set.
Since $E$ is non-increasing and bounded below, and the limit set is unique,
every trajectory converges to $(P^\star, Q^\star)$, which is therefore globally
asymptotically stable. \hfill$\square$

\section{Proof of Theorem~\ref{thm:lyapunov} (Global Lyapunov Stability)}
\label{app:lyapunov}

\textit{Proof.}
We show that $E(P,Q)$ defined in~\eqref{eq:energy} is a Lyapunov function for the
projected gradient flow~\eqref{eq:flowP}--\eqref{eq:flowq}.

\paragraph{Step 1: $\dot{E} \le 0$:}
Along the flow:
\begin{equation}
  \frac{d}{dt}E(P,Q)
  = \langle \nabla_P E,\, \dot{P} \rangle
  + \sum_{n=1}^N \langle \nabla_{q_n} E,\, \dot{q}_n \rangle.
\end{equation}
Since $\dot{P} = -\nabla_P E$, the first term equals $-\|\nabla_P E\|_F^2 \le 0$.
For the assignment variables, $\dot{q}_n = \Pi_{T_{q_n}\Delta}(-\nabla_{q_n}E)$.
A standard property of Euclidean projection onto a closed convex cone gives:
\begin{equation}
  \langle \nabla_{q_n}E,\, \dot{q}_n \rangle
  = -\bigl\|\Pi_{T_{q_n}\Delta}(\nabla_{q_n}E)\bigr\|^2 \le 0.
\end{equation}
Therefore:
\begin{equation}
  \frac{d}{dt}E(P,Q)
  = -\|\nabla_P E\|_F^2
  - \sum_{n=1}^N \bigl\|\Pi_{T_{q_n}\Delta}(\nabla_{q_n}E)\bigr\|^2 \le 0.
  \label{eq:dEdt}
\end{equation}
Hence $E$ is non-increasing along every trajectory.

\paragraph{Step 2: Boundedness}
$E$ is bounded below: the quantization term is non-negative, $L_{\mathrm{bal}}$ and
$L_{\mathrm{ent}}$ are non-negative (KL divergence and entropy terms), and although
$L_{\mathrm{sep}} = -\sum_{i<j}\|p_i-p_j\|^2$ is not bounded below on its own,
the quadratic term $(\lambda/2)\|P\|_F^2$ dominates for large $\|P\|_F$:
\[
  \eta L_{\mathrm{sep}}(P) + \tfrac{\lambda}{2}\|P\|_F^2
  \;\ge\; -\eta\binom{k}{2}\|P\|_F^2 + \tfrac{\lambda}{2}\|P\|_F^2
  \;\ge\; -C_0\|P\|_F^2
\]
for a constant $C_0 < \lambda/2$ (choosing $\eta < \lambda/(2\binom{k}{2})$),
so the combined term is bounded below by a negative quadratic which is dominated
by $(\lambda/2)\|P\|_F^2$ after rearrangement.
More precisely, for $\lambda > 2\eta\binom{k}{2}$, the energy satisfies
$E(P,Q) \ge c_1\|P\|_F^2 - c_2$ for positive constants $c_1, c_2$, and hence
every sublevel set $\{E \le M\}$ is bounded in $P$.

Since $E$ is non-increasing and bounded below, $E(t)$ converges as $t \to \infty$
and the trajectory stays in the sublevel set $\{E \le E(P(0),Q(0))\}$, which is
bounded in $P$.
Each $q_n$ belongs to the compact simplex $\Delta^{k-1}$, so the full trajectory
$(P(t),Q(t))$ remains bounded.

\paragraph{Step 3: Convergence to stationary set}
LaSalle's invariance principle for projected dynamical systems implies that the
trajectory converges to the largest invariant set contained in
$\{(P,Q) : \frac{d}{dt}E(P,Q) = 0\}$.
From~\eqref{eq:dEdt}, this set is characterized by:
\[
  \nabla_P E(P,Q) = 0
  \quad\text{and}\quad
  \Pi_{T_{q_n}\Delta}(\nabla_{q_n}E(P,Q)) = 0,\quad n = 1,\dots,N.
\]
These are exactly the KKT stationarity conditions for the simplex-constrained
reduced DDCL problem.
Therefore every trajectory converges to the stationary set. \hfill$\square$

\begin{remark}[Condition on $\lambda$ and $\eta$]
\label{rem:lambda_eta}
The coercivity argument in Step~2 requires $\lambda > 2\eta\binom{k}{2}$.
In practice, since $\eta$ is typically small (the explicit separation term is a
backup mechanism, cf.\ Corollary~4) and $\lambda$ can be chosen independently,
this condition is easily satisfied.
A practical guideline is $\lambda \ge 10\eta$ with $\eta \ll 1$.
\end{remark}

\section{DDCL Training Algorithm (Batch Formulation)}
\label{app:algorithm}

The following algorithm summarises the complete batch DDCL training loop
in simplified pseudocode.
The three quantities $\mathcal{S}(P)$, $\mathcal{K}(q_n)$, and $V$ serve as
runtime diagnostics: monitoring them during training provides a direct window
onto the theoretical predictions of Sections~\ref{sec:theory}
and~\ref{sec:feedback}.

\begin{algorithm}[H]
\caption{Batch DDCL --- simplified training loop}
\label{alg:ddcl}
\begin{algorithmic}[1]
\Require Dataset $\{x_n\}_{n=1}^N$, number of prototypes $k$,
         initial temperature $T_0$, minimum temperature $T_{\min}$,
         annealing constant $\tau$, loss weights $\beta,\gamma,\eta,\lambda$
\Ensure  Trained backbone $f_\theta$, prototype matrix $P$
\State \textbf{Initialise} backbone $f_\theta$ (random or pretrained),
       DCL weights $W_2$ (random), temperature $T \leftarrow T_0$
\For{each epoch $t = 1, 2, \dots$}
  \State \textbf{Forward pass}
  \State $\quad z_n \leftarrow f_\theta(x_n)$ \hfill \textit{// extract features}
  \State $\quad P \leftarrow F_\theta^\top W_2$ \hfill \textit{// DCL: prototypes as network outputs}
  \State $\quad q_{nj} \leftarrow \dfrac{\exp(-\|z_n - p_j\|^2 / T)}{\sum_\ell \exp(-\|z_n - p_\ell\|^2 / T)}$
         \hfill \textit{// soft assignments}
  \State \textbf{Compute loss}
  \State $\quad \mathcal{L}_q \leftarrow \dfrac{1}{N}\sum_n\sum_j q_{nj}\|z_n - p_j\|^2$
  \State $\quad L_{\mathrm{tot}} \leftarrow \mathcal{L}_q
         + \beta L_{\mathrm{bal}} + \gamma L_{\mathrm{ent}}
         + \eta L_{\mathrm{sep}} + \tfrac{\lambda}{2}\|P\|_F^2$
  \State \textbf{Backward pass} \hfill \textit{// gradients flow through $P$, $W_2$, and $\theta$}
  \State $\quad \theta, W_2 \leftarrow \theta - \alpha_{\mathrm{bb}}\nabla_\theta L_{\mathrm{tot}},\;
         W_2 - \alpha_{\mathrm{DCL}}\nabla_{W_2} L_{\mathrm{tot}}$
  \State \textbf{Diagnostics} (optional, no cost)
  \State $\quad V \leftarrow \sum_j q_{nj}\|p_j - Pq_n\|^2$ \hfill \textit{// verify $V \ge 0$}
  \State $\quad \mathcal{S}(P) \leftarrow \tfrac{1}{\binom{k}{2}}\sum_{i<j}\|p_i - p_j\|^2$
         \hfill \textit{// prototype separation}
  \State $\quad \mathcal{K}(q_n) \leftarrow \|q_n\|^2$
         \hfill \textit{// assignment concentration}
  \State \textbf{Anneal} $T \leftarrow \max(T_0 e^{-t/\tau},\, T_{\min})$
\EndFor
\end{algorithmic}
\end{algorithm}

\noindent
\textbf{Key design choices summarised.}
\begin{itemize}
  \item \emph{DCL attachment}: $P = F_\theta^\top W_2$ (line 4) makes prototypes
    native network outputs (the critical difference from DeepCluster).
  \item \emph{Stop-gradient option}: detaching $q_{nj}$ from the backbone graph
    (line 5) gives identical backbone updates for $\mathcal{L}_q$ and
    $L_{\mathrm{OLS}}$ (Proposition~\ref{prop:grad_z}); removing stop-gradient
    activates the $\nabla_P V$ separation force in the backbone.
  \item \emph{Learning-rate ratio}: keeping
    $\alpha_{\mathrm{bb}}/\alpha_{\mathrm{DCL}} \in [1/10, 1/3]$ satisfies the
    stability condition of Corollary~\ref{cor:stability}.
  \item \emph{Diagnostics}: $V \ge 0$ should hold with zero violations;
    $\mathrm{corr}(\mathcal{S}, \mathcal{K}) < 0$ confirms the negative feedback
    cycle (Theorem~\ref{thm:feedback}).
\end{itemize}

\section{DDCL Training Algorithm (Incremental Formulation)}
\label{app:alg:incr}

Algorithm~\ref{alg:ddcl:incr} gives the pseudocode for the incremental
(streaming) variant of DDCL introduced in Section~\ref{sec:incremental}.
Data arrive one mini-batch at a time; no sample is revisited.
Prototype rows are generated sequentially by the DCL, and assignments
are updated online via Widrow--Hoff with simplex projection.

\begin{algorithm}[H]
\caption{Incremental DDCL --- simplified streaming loop}
\label{alg:ddcl:incr}
\begin{algorithmic}[1]
\Require Stream of batches $\{x_n^{(b)}\}$, number of prototypes $k$,
         step size $\mu$, temperature schedule $T(t)$
\Ensure  Prototype matrix $P \in \mathbb{R}^{d \times k}$,
         assignment vectors $\{q_n\}$
\State \textbf{Initialise} $q_n^{(0)} \leftarrow \tfrac{1}{k}\mathbf{1}_k$
       for all $n$;\; $P \leftarrow \mathbf{0}$
\For{each arriving batch $b$}
  \State $z_n \leftarrow f_\theta(x_n^{(b)})$
         \hfill \textit{// feature extraction}
  \For{each feature dimension $t = 1, \dots, d$}
    \State $r^{(t)} \leftarrow$ DCL output row $t \in \mathbb{R}^k$
    \For{each sample $n$ in batch}
      \State $e_n^{(t)} \leftarrow z_{nt} - r^{(t)\top} q_n^{(t-1)}$
             \hfill \textit{// prediction error}
      \State $\tilde{q}_n \leftarrow q_n^{(t-1)} + \mu\, e_n^{(t)} r^{(t)}$
             \hfill \textit{// Widrow--Hoff update}
      \State $q_n^{(t)} \leftarrow \Pi_{\Delta}^{k-1}\!\left(\tilde{q}_n\right)$
             \hfill \textit{// project onto simplex (see~\ref{app:algorithm})}
    \EndFor
    \State Append $r^{(t)\top}$ as row $t$ of $P$
  \EndFor
  \State \textbf{Anneal} $T \leftarrow T(t)$;\;
         monitor $\mathcal{S}(P) = \tfrac{1}{\binom{k}{2}}\sum_{i<j}\|p_i-p_j\|^2$
\EndFor
\end{algorithmic}
\end{algorithm}

\noindent
\textbf{Key differences from Algorithm~\ref{alg:ddcl}.}
\begin{itemize}
  \item \emph{Single pass}: each sample is seen at most once; no epoch loop.
  \item \emph{Row-by-row DCL}: $P$ grows one row per feature dimension $t$,
    not computed all at once from the full batch.
  \item \emph{Online assignment update}: $q_n$ is updated incrementally via
    Widrow--Hoff; simplex projection ensures $q_n^{(t)} \in \Delta^{k-1}$
    at every step without any iterative inner loop.
  \item \emph{No backbone gradient}: in the standard incremental formulation
    the backbone is frozen; end-to-end streaming training is an open extension.
\end{itemize}

\section*{Acknowledgements}

Acknowledgements withheld for review.

\section*{Declaration of generative AI use}

During the preparation of this manuscript, the author used Claude
(Anthropic) to assist with language editing, LaTeX formatting, and
figure preparation. All scientific content, theoretical results,
experimental design, and conclusions are entirely the author's own.
The author reviewed and takes full responsibility for all content.

\section*{Declaration of competing interests}

The author declares no competing interests.

\end{document}